\documentclass[twoside]{article} 
\usepackage[accepted]{aistats2021}
\pdfoutput=1
\RequirePackage[ruled]{algorithm}
\RequirePackage[noend]{algpseudocode}
\usepackage{etoolbox}
\usepackage{hyperref}

\RequirePackage{amssymb}
\RequirePackage{graphicx}
\usepackage{color-edits}
\addauthor{ka}{red}   % ka for Karthik
\usepackage{subcaption}
\usepackage{graphicx}
\graphicspath{{Figs/}} %Setting the graphicspath

\makeatletter
\def\therule{\makebox[\algorithmicindent][l]{\hspace*{.5em}\vrule height .75\baselineskip depth .25\baselineskip}}%

\newtoks\therules% Contains rules
\therules={}% Start with empty token list
\def\appendto#1#2{\expandafter#1\expandafter{\the#1#2}}% Append to token list
\def\gobblefirst#1{% Remove (first) from token list
  #1\expandafter\expandafter\expandafter{\expandafter\@gobble\the#1}}%
\def\LState{\State\unskip\the\therules}% New line-state
\def\pushindent{\appendto\therules\therule}%
\def\popindent{\gobblefirst\therules}%
\def\printindent{\unskip\the\therules}%
\def\printandpush{\printindent\pushindent}%
\def\popandprint{\popindent\printindent}%

\algdef{SE}[WHILE]{While}{EndWhile}[1]
  {\printandpush\algorithmicwhile\ #1\ \algorithmicdo}
  {\popandprint\algorithmicend\ \algorithmicwhile}%
\algdef{SE}[FOR]{For}{EndFor}[1]
  {\printandpush\algorithmicfor\ #1\ \algorithmicdo}
  {\popandprint\algorithmicend\ \algorithmicfor}%
\algdef{S}[FOR]{ForAll}[1]
  {\printindent\algorithmicforall\ #1\ \algorithmicdo}%
\algdef{SE}[LOOP]{Loop}{EndLoop}
  {\printandpush\algorithmicloop}
  {\popandprint\algorithmicend\ \algorithmicloop}%
\algdef{SE}[REPEAT]{Repeat}{Until}
  {\printandpush\algorithmicrepeat}[1]
  {\popandprint\algorithmicuntil\ #1}%
\algdef{SE}[IF]{If}{EndIf}[1]
  {\printandpush\algorithmicif\ #1\ \algorithmicthen}
  {\popandprint\algorithmicend\ \algorithmicif}%
\algdef{C}[IF]{IF}{ElsIf}[1]
  {\popandprint\pushindent\algorithmicelse\ \algorithmicif\ #1\ \algorithmicthen}%
\algdef{Ce}[ELSE]{IF}{Else}{EndIf}
  {\popandprint\pushindent\algorithmicelse}%
\algdef{SE}[PROCEDURE]{Procedure}{EndProcedure}[2]
   {\printandpush\algorithmicprocedure\ \textproc{#1}\ifthenelse{\equal{#2}{}}{}{(#2)}}%
   {\popandprint\algorithmicend\ \algorithmicprocedure}%
\algdef{SE}[FUNCTION]{Function}{EndFunction}[2]
   {\printandpush\algorithmicfunction\ \textproc{#1}\ifthenelse{\equal{#2}{}}{}{(#2)}}%
   {\popandprint\algorithmicend\ \algorithmicfunction}%
\makeatother

\usepackage{comment}
\RequirePackage{amsthm,amsmath}
\usepackage{mathtools}
\newtheorem{theorem}{Theorem}

\newtheorem{definition}[theorem]{Definition}
\newtheorem{proposition}[theorem]{Proposition}
\newtheorem{corollary}[theorem]{Corollary}

\newtheorem{lemma}[theorem]{Lemma}

\allowdisplaybreaks
\makeatletter

\makeatletter
\newcommand{\printfnsymbol}[1]{%
  \textsuperscript{\@fnsymbol{#1}}%
}
\makeatother

\newcommand*{\inlineequation}[2][]{%
  \begingroup
    \refstepcounter{equation}%
    \ifx\\#1\\%
    \else
      \label{#1}%
    \fi
    \relpenalty=10000 %
    \binoppenalty=10000 %
    \ensuremath{%
      #2%
    }%
    ~\@eqnnum
  \endgroup
}
\makeatother

\usepackage{lipsum}

\newcommand\blfootnote[1]{%
  \begingroup
  \renewcommand\thefootnote{}\footnote{#1}%
  \addtocounter{footnote}{-1}%
  \endgroup
}

\usepackage[utf8]{inputenc}

\newcommand{\fullname}{UCB with Decentralized Dominant-arm Deletion}
\newcommand{\name}{UCB-D3}

\begin{document}

\twocolumn[
\aistatstitle{Dominate or Delete: Decentralized Competing Bandits in Serial Dictatorship}

\aistatsauthor{ Abishek Sankararaman$^*$ \And Soumya Basu$^*$ \And  Karthik Abinav Sankararaman }

\aistatsaddress{ AWS AI, Palo Alto, USA\footnotemark[1] \And  Google, Mountain View, USA\footnotemark[2] \And Facebook, Menlo Park, USA } ]

\blfootnote{\textsuperscript{*}Equal Contribution}
\footnotetext[1]{Part of work done while affiliated with UC Berkeley}
\footnotetext[2]{Part of work done while affiliated with UT Austin}
\stepcounter{footnote}
\stepcounter{footnote}
\begin{abstract}
 Online learning in a two-sided matching market, with demand side agents continuously competing to be matched with supply side (arms), abstracts the complex interactions under partial information on matching platforms (e.g. UpWork, TaskRabbit). We study the decentralized serial dictatorship setting, a two-sided matching market where the demand side agents have unknown and heterogeneous valuation over the supply side (arms), while the arms have known uniform preference over the demand side (agents). We design the first decentralized algorithm -- UCB with Decentralized Dominant-arm  Deletion (UCB-D3), for the agents, that does not require any knowledge of reward gaps or time horizon. UCB-D3 works in phases, where in each phase, agents delete \emph{dominated arms} -- the arms preferred by higher ranked agents, and play only from the non-dominated arms according to the UCB. At the end of the phase, agents broadcast in a decentralized fashion, their estimated preferred arms through {\em pure exploitation}. We prove both, a new regret lower bound for the decentralized serial dictatorship model, and that UCB-D3 is order optimal.% achieves order optimal regret guarantee.
\end{abstract}

\section{INTRODUCTION}
\label{sec:intro}
Online matching markets (e.g. UpWork and Mechanical Turk) are economic platforms that connect demand side, (e.g. businesses in Upwork or Mechanical Turk), to the supply side (e.g. freelancers in Upwork, or crowdworkers in Mechanical Turk). These platforms enable the demand side agents (a.k.a. agents) to make repeated decisions to match with the supply side agent (a.k.a. arms) of their preference. On these platforms, supply side agents  when faced with multiple offers, chooses the demand side agent of her choice. With resource limited supply side, the agents thus compete for arms, while navigating uncertainty on the quality of each arm. Uncertainty in matching markets have recently been studied in various disciplines, \emph{e.g.,}  \cite{matching_empirical} uses an empirical approach and  \cite{match_learn} uses an economic approach, and  \cite{matching_bandits} formalizes the problem from a learning perspective. In particular,  Liu. et.al~\cite{matching_bandits} introduced the centralized multi-agent matching bandit model, where at each time, every agent (demand side) shares her learned preference {\em truthfully} over {\em all the arms} (supply side agents) with a central arbiter, who then makes an allocation.

In this paper, we initiate the systematic study of the {\em decentralized paradigm} of this model, where no central arbiter exists, the knowledge of gaps and time horizon is unknown, and each agent relies only on her own observations to make arm choices. The decentralized paradigm, without central control, is of utmost importance in multi-agent setting, as sharing observation with a central arbiter is prone to privacy breach, lacks transparency of the arbiter, and  is susceptible to untruthful inputs from agents \cite{law-spociety,experiment-truth}. Further, in such systems any global information (e.g. minimum gap between arm rewards for agents, time horizon) is unavailable a priori.\footnote{In~\cite{matching_bandits-old} the decentralized setting was briefly mentioned for the matching bandit model, but no solution was proposed when global information is unavailable.}

We focus on the \emph{Serial Dictatorship} model (made precise in the sequel), a  subclass of the matching bandits model which is well studied in economics, game-theory and matching markets~\cite{abdulkadirouglu1998random, bogomolnaia2001new, bade2020random, aziz2013computational}. In particular, this setting has attracted a lot of interest, both in theory and practice since this is the only known mechanism that is both truthful (to elicit incentives) and Pareto optimal (to make allocations) in a two sided market ~\cite{zhou1990conjecture}.

\textbf{Model Overview.} A serial dictatorship model consists of $N$ agents and $K \geq N$ arms. 
For each agent the arms (supply side agents) are ranked {\em heterogeneously}, in increasing order of arm-means which is different for each agent-arm pair. The agents are ranked {\em homogeneously} across all arms (uniform valuation). 
%The preference ordering of an agent over arms is given by the set of arm-means  for that agent, higher the mean, higher the preference.The arm-means can be different for different agents (heterogeneous system).
Agents do not know the arm-means (and thus their preference) and have to learn them over time.
All agents choose an arm each simultaneously in each round (a decentralized system).
Thus, in every round, each arm is chosen by any number ($0$ to $N$) of agents.
An arm matches only with the highest ranked proposing agent (if any), while blocking the remaining agents (if any). Therefore, each agent is  either \emph{matched} to her arm of choice and receives a corresponding stochastic reward, or is \emph{blocked} by her arm of choice (which they are notified of) and receive a deterministic $0$ reward.  

\textbf{Objective.} The serial dictatorship system is said to be in {\em equilibrium}, when the matching between arms and the agents is the  {\em unique stable matching} \cite{uniq_matching}. %Here, a matching corresponds to an allocation of a subset of arms to agents, where each arm is allocated to at-most one agent, and every agent is allocated an arm. 
Here, a matching is stable, if there exists no agent-arm pair, who would mutually prefer each other as opposed to their current partners in the matching \cite{gale_shapley}. 
%There exists an unique stable matching in a serial dictatorship  setting \cite{uniq_matching}.
The {\em regret} of an algorithm is defined as the expected difference of cumulative reward attained by the algorithm, and the cumulative reward in equilibrium when agents and arms always match according to the {unique stable matching}. The objective is to design a decentralized algorithm for the agents that has {\em regret} sub-linear (preferably logarithmic)  in time horizon.

\textbf{Applications.}
Online learning in a serial dictatorship setting finds its application in pareto-optimal allocation of resources among agents, where the objective is to ensure that the satisfaction of an agent cannot be increased without decreasing satisfaction of another agent. A canonical use case is that of scheduling jobs to servers in an online marketplace, \emph{e.g.,} scheduling in datacenters (AWS, Azure) \cite{dickerson2019online,even2009scheduling,brucker1999scheduling}, crowdsourcing platforms (Upwork and TaskRabbit) \cite{massoulie2016capacity, basu2019pareto}, question answering platforms (Quora, Stack Overflow) \cite{shah2020adaptive}. In such marketplaces, the rewards are often stochastic and only observable upon completion. Moreover, these platforms solve thousands of such instances every hour through the numerous repeated interactions and thus would like to optimize the cumulative reward over a long time horizon. Finally, for scalability, truthfulness, and privacy-concerns  decentralized solutions are preferred, where jobs and servers interact without a central authority. 

As a concrete example, let us consider a matching platform with $N$ workers and $K$ tasks (or more specifically task types), where workers need to be matched with realized tasks (drawn from the task types). This can be captured by the repeated matching model proposed in this paper. Each worker requests their top-choice task (\emph{e.g.,} based on the monetary return for the task), and using a serial dictatorship algorithm (\emph{e.g.,} a universal rating of workers) the allocation is made to the workers. In particular, among all the offers received for a task the highest rated worker is assigned the task. The reward of a worker and task pair can only be truly realized after making the assignment. In this platform, each worker aims for maximizing their own reward.

\textbf{Why are Decentralized Algorithms hard?} One natural attempt to design a decentralized algorithm for this problem is for each agent to play the UCB1 algorithm \cite{auer-ucb} independently (without coordination). When an agent gets blocked, it treats the realized reward as $0$. However, this approach fails for many canonical examples. Consider a setting with two agents and two arms. Let the mean rewards for agent $1$ be $1$ and $1/2$ on arm 1 and arm 2, respectively. For agent $2$, let the mean rewards be $1$ and $\epsilon$ for arms 1 and 2, respectively. Arm 1 and 2 both prefers agent 1 over agent 2. Note that the reward gap for agent $1$ is $\Delta_1 = 1/2$ and agent $2$ is $\Delta_2 = 1-\epsilon$. In this case, agent 1 plays arm 1 for $(t-\mathcal{O}(\log (t)))$ time in $t$ steps, and thus agent 2 perceives the arm 1 as having mean $0$. Therefore, an application of the instance dependent regret lower bound for bandits (Theorem 16.2 in~\cite{book1}) shows the regret for such naive decentralized UCB1 algorithm is $\Omega\left(\tfrac{\log (T)}{\epsilon}\right)$. For $\epsilon \approx 0$, this bound is arbitrarily worse than the $\mathcal{O}(\log(T))$ regret upper bound of our proposed algorithm (see, Section~\ref{sec:mainresult}). The key challenge in this problem is that the agents need to have a \emph{coordinated back-off mechanism without explicit communication}. The main contribution in this paper is a novel procedure that provably achieves optimal regret (order-wise) in a $N$ agent, $K$ arm system for $K \mathtt{\geq} N \mathtt{\geq} 1$.

%\subsection{Main Contributions} 
%We detail our algorithmic, and technical contributions below. 

\textbf{1. Algorithmic Contributions.} 
%Our algorithm \fullname\,(\name);  that proceeds in phases of exponentially increasing length, introduces two new ideas - 

\textbf{(i) Non-Monotone Arm-elimination}: We introduce the idea of non-monotone phase based arm-elimination in decentralized matching bandits problem. Specifically, in each phase, each agent deletes the arms that are estimated to be the stable match partner of higher ranked agents, a.k.a. {\em dominated arms}. The intuition is as follows -- under an algorithm with low regret, agents will play its stable match partner arm, for a large fraction of time. Thus, any agent through this deletion avoids collisions, at an arm that is the true stable match partner of a higher ranked arm. Furthermore, unlike all classical algorithms such as Successive Rejects~\cite{successive_rejects}, the arm-deletion is \emph{non-monotone}. Namely, the set of deleted arm in the past is not a subset of the set of deleted arms in the future. 
We show if an agent incorrectly deletes a non-dominated arm, it will, eventually with probability $1$ under our algorithm, reverse the deletion and play this arm. Thus, non-monotonicity ensures linear regret is avoided.
%This non-monotonicity in the set of arms deleted across agents is the crucial ingredient, leading to order optimal regret in our algorithm. {\color{red} shorten a bit if possible}

\textbf{(ii) Decentralized Dominated Arm Detection}: 
We introduce a decentralized communication scheme that ensures in finite time each agent with high probability deletes all its dominated arms (if any) in a serial dictatorship. After deleting arms at the beginning of a phase, agents play according to the standard UCB in a phase by ignoring collisions. At the end of the phase, every agent sets the arm to which it was matched the most in the current phase, as its estimated stable match partner. Agents then in a decentralized fashion, communicate this through a simple pure exploitation strategy. The arms communicated will be deleted by agents in the next phase.

\textbf{2. Technical Contributions.}

\textbf{(i) Order Optimal Regret Bounds}: We prove in Section~\ref{sec:mainresult},  that, the regret of any agent ranked $j\geq 2$ over a time horizon $T$,  scales as $O \left( \log(T)/\Delta^2 \right)$, where $\Delta$ is the smallest arm-gap across all agents. A precise regret bound involving all arm means is given in Theorem \ref{thm:main_thm_algo}. In the decentralized setting, the dependence of our regret upper bound on $\Delta$ and $T$  matches with the centralized setting as given in Theorem 1 of \cite{matching_bandits}. Our results rely on a new {\em inductive dominated-arm locking} technique mentioned above. 

We further show in Section~\ref{sec:lower_bound}, through a matching lower bound, that a regret scaling $\Omega \left( \log(T)/\Delta^2 \right)$ is unavoidable. This is the first $\Omega \left( \log(T)/\Delta^2 \right)$  regret lower bound in decentralized multi-agent bandits, to the best of our knowledge. In our decentralized and heterogeneous setting, agents can not collaboratively learn the arm-means of other agents. This necessitates $\Omega(\log(T)/\Delta^2)$ exploration for any agent with a sub-optimality gap $\Delta$, which in turn may lead to  $\Omega(\log(T)/\Delta^2)$ collisions with the optimal arm for some lower ranked agent. Formalizing this intuition, we prove our instance dependent regret lower bound, and establish an order optimal regret for our algorithm.

\textbf{(ii) Inductive Dominated-arm Locking}:  We develop a new analysis idea for the decentralized serial dictatorship, that relies on inductive blindness of the agents (from highest to lowest rank). By inductive blindness we mean, for any $i\geq 1$, the agent ranked $i$ is unaffected by any other agent ranked $i+1$ or higher. Such blindness sets us apart from the literature of decentralized bandits~\cite{colab1,proutiere,sic_mmab,musical_chairs} where all agents can affect any other agents through collisions. In our algorithm, the errors in an agent's estimation of its stable match partner, \emph{leaks to the other agents through the communication block}, when an agent signals its estimated stable matched arm. As agents delete their perceived dominated arms, an error by the higher ranked agent, gets amplified as it goes downstream. Thus any error in estimating a stable matched arm can create a domino effect that propagates across agents and phases.
We prove that, if agents use the most matched arm in a phase as the estimate for its stable match partner, with probability $1$, all such erroneous cascades die, and every agent eventually, in all phases, correctly identify its stable match partner arms. Our technical contribution lies in identifying an inductive structure in the algorithm, where, with probability $1$, agents (in the order of their rank), stop spreading the wrong arm (see also Fig. \ref{fig:heatmap}).

%Our proof relies on an inductive approach, where we show agents stop providing wrong signals in the communication block, sequentially in the order of their rank (see also Figure \ref{fig:heatmap}).

\section{PROBLEM SETTING}
\label{sec:problem_setting}
%\textbf{Agents and arms.} 
We consider $N$ agents, and $K\geq N$ arms. The agents are ranked where the rank of any agent $j\in[N]$ is $j$ (which is unknown to the agents).
This is without loss of generality, as we can relabel the agents to obtain this.
At each time, all agents choose one of the $K$ arms simultaneously, to play and collect a reward.
An agent is \emph{matched} to the arm of its choice in a given round, if and only if it is the highest ranked agent playing that arm at that time.\footnote{We use the notion that rank $l'$ is higher than rank $l$ if and only if $l<l'$, throughout the paper.} Otherwise, it is \emph{blocked}.
If agent $j \in [N]$ is matched with arm $k \in [K]$ at any time, then, agent $j$ is receives a stochastic $[0,1]$ valued reward with mean $\mu_{jk} \in (0,1)$ independent of all other rewards.\footnote{This is done for convenience. Our analysis can be easily adapted to any sub-gaussian reward.} If an agent is blocked, then it is notified and receives a deterministic reward of $0$. 

The arm means $(\mu_{jk})_{j\in [N],k\in[K]}$ are heterogeneous across the agents, and are not known to the agents apriori. Furthermore, for every agent $j$, the set of $K$ arm-means $(\mu_{jk})_{k\in[K]}$ are all distinct.
In the sequel, we denote by $I^{(j)}(t) \in [K]$, to be the arm played by agent $j$ in round $t$.
For each arm $k \in [K]$ and time $t\geq 1$, denote by $M_k(t) = \min \{j : I^{(j)}(t) = k\}$, to be the agent with which it is matched, where the minimum of an empty set is defined to be infinity.

\textbf{Decentralized algorithms.}
We consider decentralized algorithms, namely, at each time $t$, the choice of arm to choose by any agent, must only depend on the events (past arm choices, rewards and blocking) observed by the agent. At the beginning, the agents are allowed to form a protocol to follow while playing.

\textbf{Unique stable matching and regret.}
In our setting, each agent \emph{prefers} arms in the increasing order of arm-means, and each arm \emph{prefers} agents according to the uniform agent ranking. 
A matching of agents to arms is stable, if there is no agent-arm pair, unmatched in the current matching, that mutually prefer each other than their respective current matches \cite{gale_shapley}. In our system, as the agents have uniform rank across all arms there exists a unique \emph{stable matching} (which is not true when agents are ranked non-uniformly across arms.) Indeed, in any stable match agent $j$ must match with it's most preferred arm which is not matched with any agent with rank $(j-1)$ or higher.\footnote{Agents ranked $1$ through $j-1$}

We compare the performance of any decentralized online learning strategy to an oracle, in which the agents and arms are matched according to this unique \emph{stable matching}. Stable matching is the appropriate benchmark, as it captures the equilibrium of the allocation of arms to agents, when agents are myopic and want to maximize individual rewards as in our setup. (See also in~\cite{matching_bandits}). Agent ranked $1$ prefers the arm $k^{(1)}_* = \arg\max_{k \in [K]}\mu_{1k}$ the most, therefore $(1, k^*_1)$ forms a stable match as arm $k^{(1)}_*$ also prefers agent $1$ the most. Now, for any agent ranked $2 \leq j \leq N$, the stable match denoted by $k^{(j)}_*$, is defined inductively as $k^{(j)}_*= \arg\max_{k \in [K]\setminus \{k^{(1)}_*,\cdots,k^{(j-1)}_* \}} \mu_{jk}$. 
In words, the stable match for agent ranked $j$ is the best arm from among all arms that do not form stable match to agents ranked $1$ through to $j-1$.
Recall that $I^{(j)}(t) \in [K]$ denotes the arm chosen by agent $j$ in time $t$.
The regret of any agent $j \in [N]$, after $T$ time steps is  
${R}_T^{(j)} = \sum_{t=1}^T  \mathbb{E} [\mu_{jk^{(j)}_*} - \mu_{j I^{(j)}(t)} \mathbf{1}_{{M_{I^{(j)}(t)}(t)} = j}].$

%We note that in what follows due to our assumption that all arm means are unique there is no tie.
%We now describe the unique stable matching formally. 

\section{UCBD3 Algorithm}
\label{sec:algorithm}

\textbf{Algorithm Overview.}
 \name; proceeds in phases, where, in each phase, each agent only plays from a subset of arms, we call \emph{active arms}. The set of active arms are fixed for an agent in a particular phase. However, the active arm sets are different across agents, and are also \emph{non-monotone}, for the same agent across time.
 %with the set of active arms chosen at the start of a phase and fixed in the duration of a phase. 
Each phase $i \in \mathbb{N}$ is split into two blocks, {\em (i)} a regret minimization block lasting $2^{i-1}$ rounds and a {\em (ii)} communication block lasting a constant $(N-1)K$ rounds. During regret minimization, all agents play from among their active set using the standard UCB algorithm \cite{auer-ucb} and ignoring collisions.
During the communication block agents communicate their estimated stable match partners through collisions. 
An agents estimate of its stable match partner arm is the \emph{one with which it was matched the most number of times in a phase}.
The active arms for an agent in the next phase is all the arms except those that were estimated to be stable-match partner to other higher ranked agents \emph{in that phase}.
\name\; is \emph{non-monotone}; even if an agent deletes an arm in a given phase, it can potentially be active in a future phase, if in the future phase, no higher ranked agent estimates this arm to be their stable match partner.

\textbf{Notation.} 
%In order to describe the algorithm, we set some notation. 
For an agent with rank $j \in [N]$, arm $k \in [K]$ and time slot $t \in \mathbb{N}$, denote by $N_k^{(j)}(t) \in \mathbb{N} \cup \{0\}$, to be the number of times agent with rank $j$ was matched to arm $k$, upto and including time $t$. If $N_k^{(j)}(t) > 0$, denote by $\widehat{\mu}_{k}^{(j)}(t)$ to be the empirical observed mean of arm $k$ by agent $j$ using all the samples upto and including time $t$.

The first $(N-1)$ time slots are used to estimate agent's ranks and subsequently \name \ proceeds in phases. 
%We show in the sequel that every agent perfectly estimates andthus, without loss of generality, we describe the algorithm assuming every agent is aware of its rank.
From time slots $N$ on wards, the algorithm proceeds in phases, numbered $i \in \mathbb{N}$, with an agents phase being non-decreasing with time.
The algorithm is synchronous, i.e., at each time $t$, all agents are in the same phase.
The pseudocode is given in Algorithm \ref{algo:main_algo} and also described below.

%by fixing an arbitrary agent $j \in [N]$.
%Phase numbered $i$ lasts for a total of $(2^{i-1}+(N-1)K)$ time slots which is divided into two blocks \textemdash the first $2^{i-1}$ time slots consists of the regret minimization block, followed by $(N-1)K$ time slots of the communication block.
%Phase $i$ starts at time slot $2^{i-1}-1 + (i-1)(N-1)K + N$ and ends at $2^i -1 + i(N-1)K+N$, both inclusive.

% starting in time slot $2^{i-1}-1 + (i-1)(N-1)K + N$ and ending at $2^i -1 + i(N-1)K+N$, both inclusive.
% Each phase $i$ is further divided \textemdash the first $2^{i-1}$ time slots consists of the regret minimization block, followed by $(N-1)K$ time slots of the communication block.
% The duration of communication block is constant in all phases, while the regret minimization block increases exponentially. 

$\bullet$ \textbf{Rank estimation.} Rank estimation occurs for the first $(N-1)$ time slots as follows. In the first time slot, all agents will pull arm $1$. In subsequent time slots $t \in [2,N-1]$, all those agents that have never been matched in time slots $[1,t-1]$ will pull arm indexed $t$. 
Agents that were matched to some $t' \in [1,t-1]$, will play arm $t'$ (the index of the first match) at time $t$.
The estimated rank of an agent is the first time slot when it was matched. If an agent is unmatched in the first $N-1$ time slots, its rank is $N$.
One can observe using an inductive reasoning that, the estimated rank of an agent is equal to its true rank. Thus, all agents are aware of their rank after this phase. The pseudo-code {\ttfamily RANK-ESTIMATION}(\,) is given in Algorithm \ref{algo:rank-estimation} in Appendix \ref{appendix_algorithms}.

%that the estimated rank of an agent is equal to its true rank. A straightforward induction establishes that at the end of time slot $t$, all agents with rank $t$ or lower will have identified their true rank. 

% In the first time slot, since all agents pull arm indexed $1$, the highest ranked agent will get matched and thus its estimated rank equals $1$. Subsequently, in each time $t \in [2,N-1]$, all agents with rank $t$ and above, pull arm indexed $t$. Thus, the agent matched at time $t$ is that with rank $t$. Finally, the only unmatched agent is that of rank $N$. Thus, we shall assume in the sequel, that all agents are aware of their rank.

$\bullet$ \textbf{Phase $\mathbf{i}$. }  We now describe the algorithm by fixing a particular phase $i \in \mathbb{N}$ and a an agent with rank $j$. The phase $i$ starts at time $S_i \coloneqq (2^{i-1}+(i-1)(N-1)K+N-1)$. It is divided into two blocks: the {\em first block}  $[S_i,S_i + 2^{i-1} - 1]$ is the regret-minimization block lasting $2^{i-1}$ time-slots and the {\em second block} is $[S_i + 2^{i-1} - 1, S_{i+1}-1]$ is the communication block lasting $(N-1)K$ time slots.
At the beginning of phase $i \in \mathbb{N}$, associated with rank $j \in [N]$ is an active set of arms $\mathcal{A}_i^{(j)} \subseteq[K]$ with cardinality $|\mathcal{A}_i^{(j)}| \geq K+1-j$. 
(Observe that $\mathcal{A}_i^{(1)} = [K]$).
In the beginning of phase $1$, we initialize $\mathcal{A}_1^{(j)} = [K]$ for all $j \in[N]$.%, the first $(K+1-j)$ arms. 

$\diamond$ \textbf{Regret-Minimization (RM) block.} In the RM block of phase $i \in \mathbb{N}$, an agent with rank $j$, will play from among the arms in $\mathcal{A}_i^{(j)}$ according to the standard UCB-$\alpha$ rule \cite{auer-ucb}, where $\alpha \geq 2$ is a hyper-parameter. Thus, the arm played (but not necessarily matched) at time $t$ in the RM block of phase $i$ is $I^{(j)}(t) \in \arg\max_{k \in \mathcal{A}_i^{(j)}} \left( \widehat{\mu}_{k}^{(j)}(t-1) + \sqrt{\frac{2\alpha \ln(t)}{N_{k}^{(j)}(t)}} \right).$ Ties are broken arbitrarily.

At the end of the RM block of phase $i$,  agent with rank $j$, creates an estimate of its best arm denoted by $\mathcal{O}_i^{(j)} \in \mathcal{A}_i^{(j)}$, as the arm that it matched with the most number of times in the RM block of phase $i$. Formally,
$ \mathcal{O}_i^{(j)} \in \arg\max_{k \in \mathcal{A}_i^{(j)}} (N_k^{(j)}[i] - N_k^{(j)}[i-1])$, where, for any agent $j\in[N]$, phase $i \in \mathbb{N}$ and arm $k\in[K]$, $N_k^{(j)}[i]$ is the number of times arm $k$ was matched to agent $j$ in all the RM blocks up to and including phase $i$, with the convention $N_k^{(j)}[-1]=0$. 
%Ties are broken arbitrarily.

$\diamond$ \textbf{Communication block.} In the communication block of phase $i \in \mathbb{N}$, agent with rank $j\in[N]$, communicates its estimated best arm $\mathcal{O}_i^{(j)}$ to every other agent with ranks $(j+1)$ through $N$.
This is achieved by arranging arm collisions in a specific way.
The communication block, which is of duration $(N-1)K$ time slots, is further sub-divided into $N-1$ sub-blocks, each of $K$ contiguous time slots. 
In the sub-block $l \in [N-1]$, { agent with rank $l+1$ will pull all the $K$ arms once each in a round robin fashion. All other agents $l^{'} \neq l$ will play their estimated match, i.e., agent $l^{'} \neq l$, will play arm $\mathcal{O}_i^{(l')}$ in all the $K$ time slots, of the $l$th communication sub-block. Observe that agent ranked $1$ will play its estimated best arm $\mathcal{O}_i^{(1)}$ in all the $(N-1)K$ communication blocks of phase $i$.}

%In the sub-block $l \in [N-1]$ of the communication block, {\em all agents with ranks $(l+1)$  to rank $N$}, will pull all the $K$ arms, exactly once in a round robin fashion starting from arm $1$. Whereas, in the $l$-th sub-block of the communication block in phase $i$, every agent with rank $l'\in[1,l]$, will play arm $\mathcal{O}_i^{(l')}$ in all the $K$ time slots.

For agent ranked $j \geq 2$, denote by  $\mathcal{D}^{(j)}_i  \subseteq [K]$, the set of arms with which agent $j$ collides in the $j-1$th sub-block of the communication block in phase $i$.
%Thus, in the $(j-1)$-th sub-block of the communication block, the agent $j\geq 2$ gets blocked by the most played arms (in the RM block of phase $i$) for agent ranked $1$ through $(j-1)$. 
Observe from the communication protocol that $\mathcal{D}^{(j)}_i := \{ \mathcal{O}_i^{(1)},\cdots,\mathcal{O}_i^{(j-1)}\}$. The pseudo-code of the communication block, namely {\ttfamily DOMINATED-ARM-DETECTION}(), is provided in Algorithm \ref{algo:algo_comm} in Appendix~\ref{appendix_algorithms}.

$\diamond$ \textbf{Update active arms.} At the end of the communication block of a phase $i \in \mathbb{N}$, every agent $j\geq 2$ will update its active set of arms for phase $i+1$, by deleting the  arms $\mathcal{D}^{(j)}_i$, i.e. $\mathcal{A}_{i+1}^{(j)} := [K]\setminus \mathcal{D}^{(j)}_i$. Agent ranked $1$ will have all the $K$ arms active, ($\mathcal{A}_{i+1}^{(1)} = [K]$). %The set of the active arms converges to the non-dominated arm for each agent $j$ with increasing phase number $i$. 

\begin{algorithm*}
\caption{\fullname\; (\name) \; (at Agent $j \in [N]$)}
\begin{algorithmic}[1]
        \LState  $\text{{\ttfamily Rank}}^{(j)}$ $\gets$ {\ttfamily RANK-ESTIMATION}()
        \Comment{First $N-1$ arm pulls to estimate rank using Algorithm \ref{algo:rank-estimation}}

		\LState $\mathcal{C}_0 = \emptyset$ \Comment{Set of arms blocked in the beginning of phase $1$}
		\For {$i \in \{1,2,\cdots\}$} \Comment{For each Phase $i$}
		\LState $\mathcal{A}^{(j)}_i \gets [K] \setminus \mathcal{C}_{i-1}$ \Comment{Set of Active arms in phase $i$}
	
		\For {$S_i \leq t \leq {S}_i+2^{i-1}$} \Comment{The first $2^{i-1}$ times (RM block) of phase $i$ }
		\LState Play an arm 
		 $
		    I^{(j)}(t) \in \arg\max_{k \in \mathcal{A}_i^{(j)}} \left( \widehat{\mu}_k^{(j)}(t-1) + \sqrt{\frac{2\alpha \log(t)}{N_k^{(j)}(t-1)}} \right)
		$
		\EndFor
		\LState $\mathcal{O}_i^{(j)} \gets $ the most matched arm from $\mathcal{A}_i^{(j)}$ in the first $2^{i-1}$ time slots of phase $i$.
		\LState $\mathcal{C}_i \gets$ {\ttfamily DOMINATED-ARM-DETECTION($\mathcal{O}_i^{(j)},i$, $\text{{\ttfamily Rank}}^{(j)}$)}\Comment{The last $(N-1)K$ times of phase $i$ given in Algorithm \ref{algo:algo_comm}}
		\EndFor
\end{algorithmic}
\label{algo:main_algo}
\end{algorithm*}

\section{REGRET UPPER BOUND}\label{sec:mainresult}
\textbf{Dominated arms and Gaps.} For each agent $j \in [N]$, denote by the set of dominated arms for agent $j$ as $\mathcal{D}_{*}^{(j)} := \{k_*^{(1)},\cdots,k_*^{(j-1)}\}$ and the non-dominated arm as  $\mathcal{A}^{(j)}_{*}:= [K] \setminus \{k_*^{(1)},\cdots,k_*^{(j-1)}\}$, where $k_*^{(j)}$ is the stable matched arm for agent $j$.\\
 For any agent $j \in [N]$ and arm $k \neq k^{(j)}_*$, denote by $\Delta^{(j)}_k  = (\mu_{jk_*^{(j)}} - \mu_{jk})$ the gap in reward of the $k$-th arm and the stable matching arm $k^{(j)}_*$. Thus, by definition of $k^{(j)}_*$ and by the uniqueness assumption of arm-means, for all non dominated arms $k \in \mathcal{A}^{(j)}_{*} \setminus \{k^{(j)}_*\}$, we have $\Delta^{(j)}_k > 0$. Let $ \Delta := \min_{j \in [N]} \min_{k \in \mathcal{A}_{*}^{(j+1)}} \Delta_k^{(j)}$ the smallest gap of non-dominated arms across all agents. Since arm-means are unique for all agents, $\Delta > 0$.

%\textbf{Minimum gap.} Let us denote by $ \Delta := \min_{j \in [N]} \min_{k \in \mathcal{A}_{*}^{(j+1)}} \Delta_k^{(j)}$ the smallest gap of non-dominated arms across all agents. Since arm-means are unique for all agents, $\Delta > 0$.

%We have $2^{i^*} := O \left( \frac{NK}{\Delta^2} \log \left(\frac{NK}{\Delta^2} \right) \right)$ for  $i^* \mathtt{:=} \min \left\{ i \mathtt{\in} \mathbb{N}: \tfrac{20 NK \alpha i }{\Delta^2} \leq 2^{i-1} \right\}.$ This $i^*$ is the phase where the minimum gap in the system can be learned reliably. 

%We now state our main result on regret upper bound.
% \begin{theorem}
% Suppose every agent runs \name\; with $\alpha \geq 2$, then the regret of any agent with rank $j$ satisfies
% \begin{multline}
%     \textstyle\mathbb{E}[R_T^{(j)}] \leq  \underbrace{\sum_{i = 1}^{j-1}\sum_{k \in \mathcal{A}^{(j)}_*}\left(\frac{4\alpha\log(T)}{(\Delta_k^{(i)})^2} + 1\right)}_{\text{Regret due to Collisions}}\\ +\underbrace{(j-1)K \log_2(T)}_{\text{Reg from communication}}
%     + \underbrace{\sum_{k \in \mathcal{A}^{(j)}_{*}\setminus k_*^{(j)}} \left( \frac{4 \alpha \log(T)}{\Delta_k^{(j)}}\right)}_{\substack{\text{Regret due to sub-optimal}\\\text{non-dominated arm matches}}} \\
%     +N + j(2^{i^*}+i^* + 4(NK)^2 + 3NK) + \frac{2NK}{T^{2\alpha-4}},
%     \label{eqn:main_thm}
%     \end{multline}
%     where $i^* \mathtt{:=} \min \left\{ i \mathtt{\in} \mathbb{N}: \tfrac{20 NK \alpha i }{\Delta^2} \leq 2^{i-1} \right\}.$
%     \label{thm:main_thm_algo}
% \end{theorem}

\begin{theorem}
Suppose every agent runs \name\; with $\alpha \geq 2$, then the regret of any agent with rank $j$ satisfies
\begin{multline}
    \textstyle\mathbb{E}[R_T^{(j)}] \leq  \underbrace{2(j-1)(\log_2(T)+1)}_{\text{Regret due to Communication}} +  \\ \underbrace{ \sum_{j^{'}=1}^{j-1} \sum_{k \in \mathcal{A}^{*}_j}\left( \frac{9\alpha \log(T)}{(\Delta^{(j^{'})}_k)^2} + 1 +\frac{178T^{-4\alpha}}{(\Delta^{(j^{'})}_k)^2} + T^{2-4\alpha}\right)\mu_{jk^{(j)}_*} }_{\text{Regret due to Collision}} \\ + \underbrace{ \sum_{k \in \mathcal{A}^{*}_j \setminus \{ k^{(j)}_* \}}\left( \frac{9\alpha \log(T)}{(\Delta^{(j)}_k)^2} + 1 +\frac{178T^{-4\alpha}}{(\Delta^{(j)}_k)^2} + T^{2-4\alpha}\right)\Delta^{(j)}_k }_{\text{Regret due to Sub-optimal arm pull}} \\
    +N + j(2^{i^*}+i^* + 4(NK)^2 + 3NK) + \frac{2NK}{T^{2\alpha-4}},
    \label{eqn:main_thm}
    \end{multline}
    where $i^* {:=} \min \left\{ i {\in} \mathbb{N}: \tfrac{20 NK \alpha i }{\Delta^2} \leq 2^{i-1} \right\}.$
    \label{thm:main_thm_algo}
\end{theorem}

The following corollary highlights the dependencies of the regret on different model parameters.
\begin{corollary}
\label{cor:UBcor}
Suppose every agent runs \name\; with parameter $\alpha \geq  2$, then the regret of any agent with rank $j$ after $T$ time steps satisfies
\begin{align*}
\textstyle \mathbb{E}[R_T^{(j)}] \leq &9\alpha \log(T) \left(  \tfrac{(j-1)(K+1-j)}{\Delta^2} + \tfrac{[K-1-j]_{+}}{\Delta}\right) \\
&+ O \left( (NK)^2 + \tfrac{NK}{\Delta^2}\log \left(\tfrac{NK}{\Delta^2} \right) \right),
\end{align*}
where for any $x \in \mathbb{R}$, $[x]_+ := \max(x,0)$.
\end{corollary}
Corollary~\ref{cor:UBcor} implies that the best ranked agent (\emph{i.e.,} agent $1$) will experience no collisions, and thus its regret is at most $O \left( {\log(T)/\Delta} \right)$. For any other agent ranked $2 \leq j \leq N$, the regret is at most $O \left( {(K+1-j)(j-1)/\Delta^2} \log(T) \right) + O\left( {[K-1-j]_{+}/\Delta}{\log(T)} \right)$. 
{ However, this is only an \emph{upper bound} on regret. In particular, under certain circumstances the regret under our algorithm can be negative \textemdash if an agent is matched to an arm with higher mean compared to the mean of its stable match partner arm. Moreover, we observe in simulations (Figure \ref{fig:5-5-system} and Appendix \ref{appendix-simulations}) that, \name \ outperforms the only prior decentralized algorithm, the ETC from \cite{matching_bandits} and is comparable to the centralized UCB in \cite{matching_bandits}.
}

\textbf{Proof sketch of Theorem~\ref{thm:main_thm_algo}.} Using the rank estimation protocol each agent learns its own rank correctly in a decentralized way. We may assume agent $j$ knows it is ranked $j$ henceforth.
As agent $1$ experiences no collisions, we show that for all phases after a random finite number of phases $\tau^{(1)}$, agent $1$ identifies the best arm in the stable match (see also Figure \ref{fig:heatmap}).
Thus, from phase $\tau^{(1)}$ on-wards, agent $2$ will eliminate the correct arm, arm $k_*^{(1)}$.
Moreover, as agent $1$ is playing UCB algorithm, the number of times it plays a sub-optimal arm is small and thus, the total number of collisions experienced by agent $2$ is small.
Subsequently, we show that for all phases after a random phase $\tau^{(2)} \geq \tau^{(1)}$, agent $2$ always identifies its correct arm $k_*^{(2)}$ as the best arm. Similarly, we show that for all $j\in[2,N]$, for all phases after a random phase $\tau^{(j)}$, all agents $j' \leq j$ always identifies the set of dominated arms $\mathcal{D}^{(j)}_*$ (hence the non-dominated arms $\mathcal{A}^{(j)}_*$),  correctly by eliminating the most played arm of the higher ranked agents in the previous phase. Furthermore, from $\mathcal{A}^{(j)}_*$ it identifies its stable match $k_*^{(j')}$ as its best arm due to UCB dynamics, and thus, regret of all agents ranked $j$ or lower is well-behaved (see also Figure \ref{fig:heatmap}). The regret bound is proved by establishing upper bounds on the expectations of $\tau^{(j)}$, $\forall j \in [K]$.
We do so by establishing high probability upper bound on the number of times a sub-optimal arm is played by the UCB algorithm, even though the set of active arms varies across phases for each agent (except agent $1$) (Lemma \ref{lem:chi_tail_bound} in appendix).
\begin{figure}
    \centering
    \includegraphics[width=\linewidth]{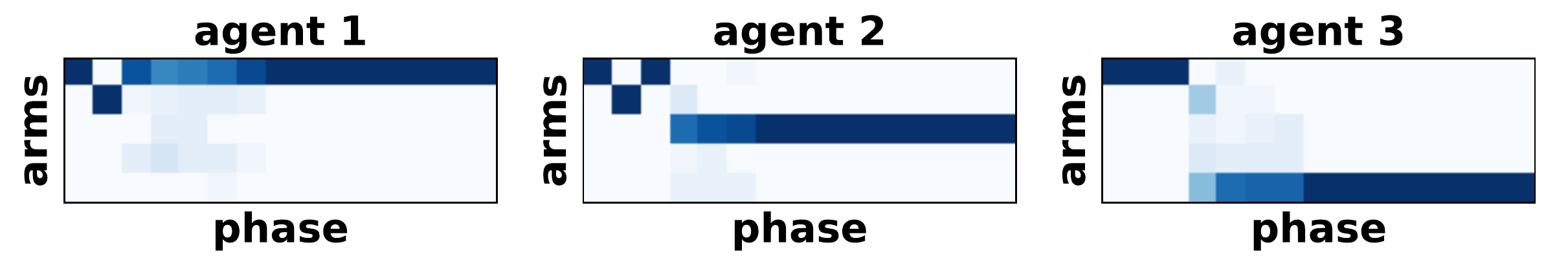}\\
    \includegraphics[width=0.67\linewidth]{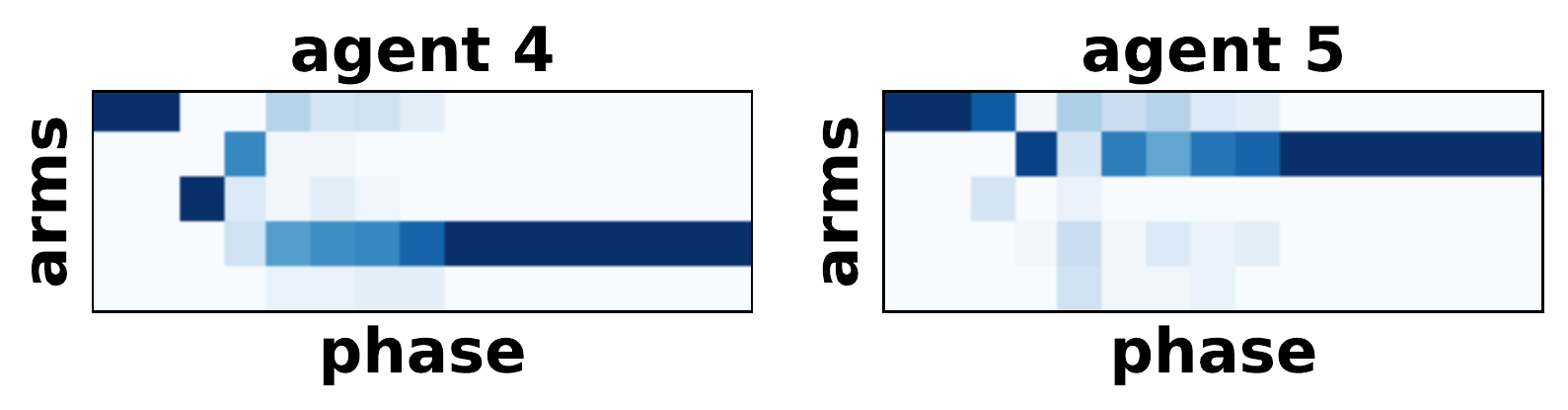}
    \caption{\small A system of $5$ agents and $5$ arms with arm-means chosen i.i.d. uniform in $[0,1]$. Over $100$ parallel runs with $13$ phases each, this heatmap counts the arm communicated by each agent at the end of each phase across runs. The color intensity increases with the number of runs where the arm is communicated. This figure demonstrates the freezing of dominated-arms empirically; all agents communicate its stable match (depicted as a contiguous black strip) partner arm after a certain phase number.}
    \label{fig:heatmap}
    \vspace{-1em}
\end{figure}

\section{Incentive Compatibility and Robustness to Selfish Players}
In this section, we discuss robustness properties when any agent deviates from executing the \name\; algorithm in order to maximize the collected reward. We consider a desirable robustness property called $\varepsilon$ Nash Equilibrium for multi-agent algorithms, recently proposed in \cite{equilibrium}. 
%- {\em (i)} $\varepsilon$ Nash Equilibrium and {\em (ii)} $(\alpha,\varepsilon)$ stability. 
% Such notions are very common in Game-Theory and are denoted as incentive compatible mechanisms.
Roughly, this property guarantees that no agent can significantly increase (by at-most additive $\varepsilon$) its rewards by unilaterally deviating from the \name; protocol.
Although this is a weaker concept compared to the classical Nash-Equilibrium used in the theory of repeated games \cite{game_theory}, is nevertheless a useful property for practical algorithms in multi-agent bandits to posses \cite{equilibrium}.
\\

We set some notations to define this concept. For any agent $j \in [N]$, time horizon $T$ and \emph{algorithm profile} $s:=(s_1,\cdots,s_N)$ executed by the $N$ agents respectively, denote by $\text{Rew}_T^{(j)}(s)$ to be the reward collected by agent $j$ in $T$ time slots, when agent $j \in [N]$ executes algorithm $s_i$. 
For an algorithm profile $s$, an algorithm $s'$ and 
any agent $j \in [N]$, denote by $\text{Rew}_T^{(j)}(s_{-j},s')$ to be the reward obtained by agent $j$ when it executes algorithm $s'$ and the other agents play according to the profile $s$.

\begin{definition}
An algorithm profile $s$ is an $\varepsilon := (\varepsilon_j)_{j=1}^N$ Nash Equilibrium if, for every agent $j\in[N]$ and algorithm $s'$, $\mathbb{E}[\text{Rew}_T^{(i)}(s_{-j},s')] \leq \mathbb{E}[\text{Rew}_T^{(j)}(s)] + \varepsilon_j.$
\end{definition}

\begin{proposition}
The \name\; algorithm profile is $\varepsilon : (\varepsilon_j)_{j=1}^N$ stable  where, for all $j\in[N]$, $\varepsilon_j = \sum_{l=1}^{j-1}\mathbf{1}_{(\mu_{jl} > \mu_{jk^*_j})}\frac{\mu_{jl}}{\mu^{(l)}}\mathbb{E}[R_T^{(l)}] + \mathbb{E}[R_T^{(j)}]$, where for all $j' \in [N]$, $\mathbb{E}[R_T^{(j')}]$ is given in Equation (\ref{eqn:main_thm}).
\label{prop:incentive}
\end{proposition}
The proof is given in Appendix \ref{appendix:proof-incentive}. This proposition gives that for agent ranked $j$, $\varepsilon_j = O \left( j \frac{\mu^{(j)}_{\text{max}}}{\mu^{(j)}_{\text{min}}} \mathbb{E}[R_T^{((j))}] \right) $, where $\mu^{(j)}_{\text{max}}$ ($\mu^{(j)}_{\text{min}}$) is the maximum (minimum) arm-mean for agent $j$. 
%The proof is in Section \ref{sec:proof-incentive}.
%This establishes that \name\; is approximately incentive compatible, namely, even if an agent deviates from the \name\; algorithm, the possible improvement in reward is $O(\log(T))$.
Roughly speaking, this proposition guarantees that even if an agent unilaterally deviates from playing the \name \ algorithm, it must still incur $O(\log(T)/\Delta^2)$ regret.

% \textbf{Incentive compatability.} We further discuss in Appendix \ref{appendix:proof-incentive} in Proposition \ref{prop:incentive}, that our algorithm satisfies certain {\em incentive compatible properties.}
% Roughly speaking, this proposition guarantees that even if an agent unilaterally deviates from playing the \name \ algorithm, it must still incur $O(\log(T)/\Delta^2)$ regret.

\section{REGRET LOWER BOUND}
\label{sec:lower_bound}
We provide a regret lower bound
by adapting the approach in \cite{auer-ucb} to our multi-agent setup.
%In this section we provide an instance dependent regret lower bound for the uniform valuation two sided market. We generalize the standard lower bound approach (\cite{auer-ucb,book1}) to the multi-agent framework with the two-sided market feedback.
%Firstly, due to the presence of blocking the dissemination of information happens differently -- blocked agents do not obtain information about the reward. Secondly, the regret decomposition also needs incorporating the structure in the unique rank stable matching problem. All the proofs are deferred to the supplementary material. 
Let $R_T(\boldsymbol{\nu}, \pi)$ denote the cumulative (sum over all agents and time horizon) expected regret of a policy $\pi$ on the instance with arm distributions $\mathbb{\nu} = \{\nu_{jk} : j\in [N], k\in [K]\}$ for a horizon of length $T$. Also, denote by $\mathcal{P}$ the set of all probability distributions with bounded support $[0,1]$.\footnotemark We define $D_{\inf}(\nu, x, \mathcal{P}) = \inf_{\nu'\in \mathcal{P}} \{D(\nu, \nu'): \mu(\nu') > x\}$ for any distribution $\nu \in \mathcal{P}$. Here,  $\mu: \mathcal{P} \to \mathbb{R}$ is the operator mapping a distribution  in $\mathcal{P}$ to its mean and $D(\cdot, \cdot)$ is the KL divergence.
\footnotetext{Please refer to the supplementary material for a formal definition of the policy and the environment.}
\begin{definition}[Uniformly Consistent Policies]
A policy $\pi$ is {\em  uniformly consistent} if and only if for all $\boldsymbol{\nu} \in \mathcal{P}$, all $\alpha \in (0,1)$, the regret  
$\limsup_{T\to \infty} \frac{R_T(\boldsymbol{\nu}, \pi)}{T^\alpha} = 0$.
\end{definition}
This notion is used to eliminate tuning a policy to the current instance while admitting large regret in other instances (c.f. \cite{auer-ucb,book1}).

\textbf{Optimally stable bandits.} Our regret lower bounds hold over a sub-class of bandits where the stable matching is optimal.  Let us consider the class of bandit instances where dominated arms are bad, i.e. for any instance $\boldsymbol{\nu}$ in this class, for all agents $j \in [N]$,  $\mu_{jk} < \mu_{jk_*^{(j)}}$ for all arms $k \in [K]\setminus \{k^{(j)}_*\}$ . We call this class of instances Optimally Stable Bandits (OSB), as each agent is matched with its optimal arm in the stable matching. Let for all $j\in [N]$, $\Delta^{(j)}_{\min} = \min\limits_{k\neq k_*^{(j)}}\Delta^{(j)}_k$, which is always non-negative for an OSB instance.
\begin{lemma}[Regret Decomposition for OSB]\label{lemm:regret}
For a OSB instance $\boldsymbol{\nu} = \{\nu_{jk} : j\in [N], k\in [K]\}$,  and any uniformly consistent policy $\pi$, agent $j\in [N]$ the following holds
\vspace{-1em}
\begin{align*}
\textstyle R_{T}^{(j)}(\boldsymbol{\nu}, \pi) \geq 
\max&\left\{\sum_{j'=1}^{j-1} \Delta^{(j)}_{\min}\mathbb{E}_{\nu, \pi}[N^{(j')}_{k_*^{(j)}}(T)], \right.\\
&\left.\sum_{k \notin \mathcal{A}^{(j)}_{*}\setminus k_*^{(j)} } \Delta^{(j)}_k \mathbb{E}_{\nu, \pi}[N^{(j)}_{k}(T)]\right\}.
\end{align*}
\vspace{-1em}
\end{lemma}
In a OSB instance, if an agent $j$'s stable match partner $k^{(j)}_*$ gets matched to an agent ranked higher than $j$, then the next bext case is agent $j$ gets matched to its second best (among all $K$ arms).
Even though this may happen rarely, it suffices to bound the first term in the above expression. Similar observations were utilized in providing a minimax (not instance dependent) lower bounds in \cite{matching_bandits}.

\begin{theorem}\label{thm:lower_bound}
For any agent $j\in [N]$, under any decentralized universally consistent algorithm $\pi$  on a OSB instance $\boldsymbol{\nu}$ satisfies 
\begin{align*} \textstyle
\liminf\limits_{T\to \infty} \frac{R_T^{(j)}(\boldsymbol{\nu}, \pi)}{\log T} 
\geq \max&\left\{ \sum_{j'=1}^{j-1} \frac{ \Delta^{(j)}_{\min}}{D_{\inf}(\nu_{j'k_*^{(j)}}, \mu_{j'k_*^{(j')}}, \mathcal{P})},\right.\\
&\left.\sum_{k\notin \mathcal{A}_{j}\setminus k_*^{(j)}} \frac{\Delta^{(j)}_k}{D_{\inf}(\nu_{jk}, \mu_{jk_*^{(j)}}, \mathcal{P})}\right\}
\end{align*}
\end{theorem} 
In the following corollary, we show that the dependence of $\log T/\Delta^2$ in the regret upper bound (Theorem~\ref{thm:main_thm_algo}) is tight up to  $O(K)$. A detailed discussion on arm-gap dependence is provided in Appendix \ref{appendix-lower-bound}.
For all agents other than the best ranked agent, the regret scales as  $\Theta\left( \log(T)/\Delta^2 \right)$ in our model. 
%This is in contrast to any (centralized or decentralized) algorithm for the multi-agent symmetric collision model \cite{sic_mmab}, where the regret for any agent scales as $\Theta\left(\log(T)/\Delta \right)$. 

\begin{corollary}\label{corr:simple}
There exists a OSB bandit instance with Bernoulli rewards, where the regert of agent $j$ is lower bounded as 
$\Omega\left(\max\left\{\tfrac{(j-1)\log T}{\Delta^2}, \tfrac{K\log T}{\Delta}\right\}\right)$.
\end{corollary}

As the reward is heterogeneous the highest ranked agents with sub-optimality gap $\Delta$ is forced to explore $\Omega(\log(T)/\Delta^2)$ time each. Whereas, all other lower ranked agents are forced to compromise a lot whenever it's sub-optimality gap is large ( $\Delta^{(j)}_{\min} =\Omega(1)$). The proof is built on the above observation. 

{\color{black}
\subsection{Comment on the Gap between the Upper and Lower Regret Bounds}
In order to improve our bounds, we hit on a fundamental roadblock, that we describe here. In the proof of Theorem \ref{thm:lower_bound},  when an agent $1$ to $j-1$ plays the stable arm  (which is also optimal in a OSB instance) arm for agent $j$, the following needs to be understood.
\begin{enumerate}
\item  Whether agent $j$ is able to ``anticipate the collision" and move to a sub-optimal arm that is collision free. If such anticipation is missing the agent accrues $\mu_{jk_*^{(j)}}$ regret (not $\Delta^{(j)}_{\min}$). 
\item Given that the agent $j$ avoids a collision, it is unclear whether the agent successfully plays the second-best arm. It is possible that it accrues  $\Delta^{(j)}_{\max}$ regret instead in the worst case.
\end{enumerate}
The above scenarios lead to the gap in the upper versus lower regret bounds. For closing such gaps it becomes essential to formalize anticipatory behavior among multi-agent decentralized bandits. This is beyond the scope of the current paper. Transfer of knowledge from lower to higher rank agents also is not possible here which makes mimicking a centralized algorithm difficult in our setting.  
}

% The arm-means for sub-optimal arms for each agent are chosen i.i.d. uniformly over $[0,0.8]$, while the arm-mean of $i \in [5]$ for agent $i$ was set to $0.9$. The rewards are binary. The plots are averaged over $100$ runs with $95\%$ confidence intervals. Our algorithm outperforms ETC in\,\cite{matching_bandits} (with $H=1056$), and is qualitatively similar to the centralized UCB in\,\cite{matching_bandits}.
%\input{comparision-prior-work}
\section{SIMULATIONS}
\label{sec:simulations}
In this section, we compare our algorithm to the centralized UCB \cite{matching_bandits} and the ETC based decentralized algorithm \cite{matching_bandits-old}. We simulate the three algorithms on a OSB instance with $5$ agents and $5$ arms, and report the regret from these algorithms in Figure \ref{fig:5-5-system}. For this instance, we first fix a random permutation $\sigma$, where the arm-mean of arm $\sigma(i)$ for agent $i$ was set to $0.9$ for all $i$. All other arm-means was chosen randomly and uniformly in $[0,0.8]$. We estimate the mean regret from $30$ independent runs. Our results clearly indicate that, our algorithm outperforms the decentralized ETC algorithm, and is comparable to the centralized UCB algorithm.
\begin{figure}[ht!]
\centering
  \includegraphics[width=\linewidth, height= 0.9in]{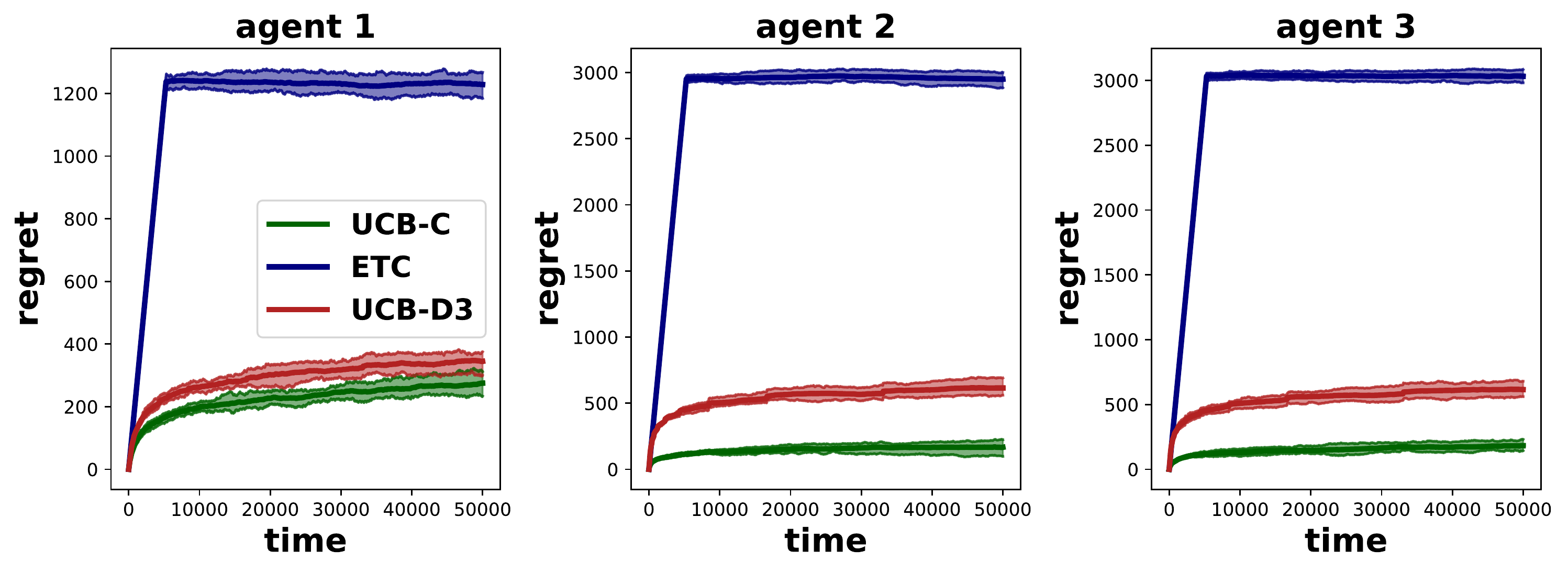}\\
  \includegraphics[width=\linewidth, height= 0.9in]{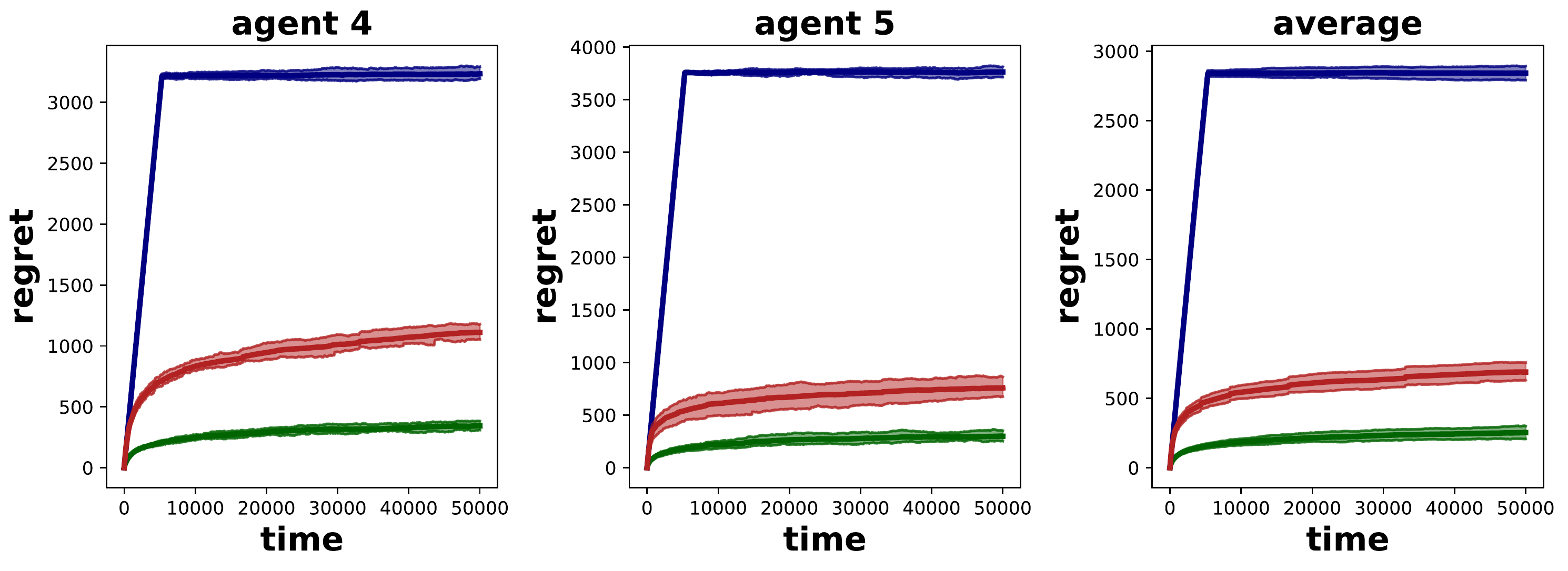}\\
\caption{Simulations with $5$ agents, and $5$ arms. We observe UCB-D3 outperforms ETC by large margin.}
\label{fig:5-5-system}
\vspace{-1em}
\end{figure}
In the appendix, we present more simulation results, and related discussions. In particular, we consider $4$ additonal systems -- the first two systems are the OSB systems with $10$ agents and $10$ arms (Fig.\,\ref{fig:10-10-system}), and $10$ agents, $15$ arms (Fig.\,\ref{fig:10_15_system}).  We also consider two non OSB systems with $5$ agents and $7$ arms (Fig.\,\ref{fig:5_7_system}), and $10$ agents and $15$ arms (Fig.\,\ref{fig:10_10_system_nonosb}). 

\section{MODEL EXTENSIONS}

We now highlight that our algorithmic paradigm is quite general, and can be potentially extended well beyond the serial dictatorship setting considered in the present paper. A well known sufficient condition for a bipartite graph to admit a unique stable matching is known as the 
Sequential Preference Condition (SPC) condition \cite{uniq_matching}, \cite{spc_cond}. Until a recent breakthrough in \cite{spc_cond}, SPC was the weakest known sufficient condition for unique matching.
The ideas introduced in this paper, can potentially be extended to provide a novel decentralized algorithm that obtains low (logarithmic) regret in the SPC setting.

Consider a marketplace with $N$ agents (indexed $1,2, \ldots N$) and $K$ arms, with $N \leq K$. Under the SPC condition, there exists a permutation of the arms $\sigma$, and the agents $\sigma'$ such that for all $i \in [N]$, (i) agent $\sigma'(i)$ prefers the arm $\sigma(i)$ the most among $[K]\setminus \{\sigma(l):l<i\}$, and (ii) arm $\sigma(i)$ prefers the arm $\sigma'(i)$ the most among $[N]\setminus \{\sigma'(l): l<i\}$. The above condition ensures, the inductive matching of agent $\sigma'(i)$ with arm $\sigma(i)$. Note that, the serial dictatorship studied in this paper is a special case of the SPC model with agent permutation $\sigma'$ as identity, and $\sigma(i)$ as the stable matched arm for agent $i$, for all $i \in [N]$.

%the system is such that there exists an arm $k_1 \in \{1,\cdots,K\}$, such that agent $1$ is the most-preferred agent for arm $k_1$ and agent 1 preferred arm $k_1$ the most. Now, for any agent $j \geq 2$, the system is such that there exists an arm $k_j$, such that agent $j$ prefers arm $k_j$, the most in the set of arms $\{1,\cdots,K\}\setminus \{k_1,\cdots,k_{j-1}\}$ and arm $k_j$ prefers the agent $j$ the most in the set $\{j,j+1,\cdots,N\}$. Note that, the serial dictatorship studied in this paper is a special case of the SPC model. 

\name\; can be adapted to this general setting as follows. The current algorithm does not work as there is no dictator (an agent which is preferred by all `unmatched' arms) at any point. However, this can be resolved by adding a local-deletion step, where agents further discard arms, that have collided a certain number of times in a phase (hence local-deletion, since deletion is in a phase). The key intuition is that under an algorithm with appropriately tuned local-deletion, agents can ``lock-in" to the equilibrium in an inductive fashion and incur negligible collisions given that the phase lengths are sufficiently large. For the base case, observe that as the agent $\sigma'(1)$ tries out an arm other than  $\sigma(1)$, it either is the preferred agent for the arm and discards the arm as it is inferior to $\sigma(1)$ (through UCB); otherwise it is blocked often at the arm and discards it through local deletion. Thus agent $\sigma'(1)$ locks in to $\sigma(1)$. Similar arguments will ensure lock in for all other agents. The appropriate tuning of local-deletion to achieve logarithmic regret in the SPC setting remains open for future work.

Finally, obtaining good decentralized algorithms remain open beyond SPC condition. This includes unique stable matching when SPC condition fails, and the general matching setup with multiple stable matching. We believe these are out of reach of the current tools developed in this paper, as the phenomenon of ``lock-in" into a stable equilibrium is missing. %The key challenge in the general matching is that there is no inductive path, to ensure that agents ``lock-in" into a specific stable equilibrium among the many.
\section{RELATED WORK}
\label{sec:related_work}
MAB are widely studied owing to a multitude of applications (\cite{book1, book2}). 
In recent times, as the scale of applications increases, multi-agent MAB problems have come into focus.
The paper \cite{matching_bandits} is the closest to ours, as it introduced the matching bandit model in more generality. 
However, the algorithms in that paper were either {\em (i)} centralized and required agents at all points of time to submit a ranked list of arms to a centralized scheduler (platform) or {\em (ii)} decentralized but needed information on minimum arm-gap across all agents and the time horizon. For the important special case of Serial dictatorship, we give a decentralized algorithm (and a lower bound), that does not need knowledge of arm-gaps or time horizon. Thus, we partially resolve an open question in \cite{matching_bandits} on decentralized algorithms for matching bandits. 

The literature on multi-agent MAB can be classified into two \textemdash \emph{competitive} where different agents compete for limited resources (as in this paper)  or \emph{collaborative}, where agents jointly accomplish a shared objective. The canonical model of competitive multi-agent bandits is one wherein if multiple agents play the same arm, they all are blocked and receive no reward (colliding bandit model) \cite{cognitive_bandit_2, kalathil,  cognitive_bandit, got, sic_mmab, practical,  musical_chairs, compete1}. 
Such models are motivated from applications in wireless networks \cite{wireless_bandit}.
However the symmetry in the problem, where if multiple agents choose the same arm, they are all blocked, are crucially used in all decentralized algorithms for that model and are hence not applicable directly to our setup. See more details in Section \ref{sec:comparison-other}. The collaborative models, consists of settings where if multiple agents play the same arm simultaneously, then they all receive independent rewards \cite{social_learn,colab1,colab2,colab3,colab4,colab5}.
Such models have primarily been motivated by applications such as internet advertising \cite{gossip_bandit}. 
However, algorithms there rely on agents collecting independent samples for arms, which are not applicable to our setting.

\subsection{Comparison with Regret Bounds for Related Models}
\label{sec:comparison-other}

The regret bound in our problem is in contrast with the performance of a related and a widely studied model known as the multi-agent 
\emph{ colliding bandits model} \cite{sic_mmab},\cite{proutiere} where, if two or more agents pull the same arm, then all agents are blocked.
The regret in this case is measured with respect to an optimal allocation of agents to arms by an oracle that knows all the arm means.
The per-agent regret for the colliding bandit model, for both the centralized and decentralized algorithms scale as $O \left( \frac{\log(T)}{\Delta} \right)$.
The gap in performance between our model and the colliding bandits model
arises due to the \emph{asymmetric collisions}; when multiple agents choose the same arm, not all of them experience a collision in our model.
{Thus, agents cannot infer if its actions cause a collision to other higher ranked agents and hence the regret must scales as $\Omega \left( \frac{\log(T)}{\Delta^2} \right)$.
	In contrast, the colliding bandit model is symmetric; if multiple agents pull the same arm simultaneously, then all of them get blocked.
	Thus, the agents can coordinate in a decentralized way \cite{proutiere} to eliminate blocking and obtain a regret of $O \left( \frac{\log(T)}{\Delta} \right)$.}
\section{CONCLUSION}
We considered the heterogeneous multi agent matching bandit problem and proposed \name\; a novel decentralized algorithm with optimal regret guarantees. This proceeds in phases, and in each phase, an agent `eliminates' arms it will likely collide. 
The main insight from our algorithm was that if agents delete their dominated arms (arm optimal for higher ranked agents), then they will incur low regret. 
%Our algorithm is non-monotone and needs to be constantly re-evaluated, to yield low regret.
%Nevertheless, agents need to constantly re-evaluate these arms to be deleted.
%We show, both by analyzing our algorithm and by exhibiting a lower bound that the regret of any agent with rank $2$ and above decays as $O\left(\frac{\log(T)}{\Delta^2} \right)$.
%This is in contrast to the centralized setting, where regret scales as $O\left(\frac{\log(T)}{\Delta} \right)$. 
An interesting avenue for future work is to devise decentralized algorithms in the non-serial dictatorship setting.

\bibliographystyle{apalike}
\bibliography{main-aistats}

\newpage
\onecolumn
\appendix
\section{Sub-Routines used in  Algorithm \ref{algo:main_algo}}
\label{appendix_algorithms}

Here, we provide the pseudo code, related to initial rank estimation and the communication protocol  used in Algorithm \ref{algo:main_algo}. 
%These are sub-routines related to initial rank estimation and the communication protocol followed by agents. 

\begin{algorithm*}
		\caption{RANK-ESTIMATION (at agent $j$)}
\begin{algorithmic}
				\LState \textbf{Initialization}:  {\ttfamily Rank}$\gets N$, {\ttfamily Flag}$\gets$ \textbf{FALSE} 
				\For {$1\leq t \leq N-1$} \Comment{Rank Estimation}
				\If {$t==1$ \textbf{OR} {\ttfamily Flag}== \textbf{False}}
				\LState $I_j(t) = t$ \Comment{Play arm $t$ at time $t$}
				\If{Matched at time $t$, i.e., $M_t(t) = j$}
				\LState {\ttfamily Rank}$\gets t$, {\ttfamily Flag}$\gets$ \textbf{TRUE}
				\EndIf
				\Else
				\LState $I_j(t) = ${\ttfamily Rank}
				
				\EndIf
				\EndFor 
				\State \Return {\ttfamily Rank}
\end{algorithmic}
\label{algo:rank-estimation}
\end{algorithm*}

\begin{algorithm*}
	\caption{DOMINATED-ARM-DETECTION ($\mathcal{O},i$, {\ttfamily Rank}) (at agent $j$)}
\begin{algorithmic}
\State \textbf{Input} $\mathcal{O} \in [K]$ - Arm to communicate, $i \in \mathbb{N}$ - the phase and {\ttfamily Rank} - the rank of the agent
                \LState $\mathcal{C} \gets \emptyset$
                \LState $\widehat{S}_i \gets S_i + 2^{i-1}+1$
				\For {$\widehat{S}_i \leq t < \widehat{S}_i + K\max(0,\text{{\ttfamily Rank}-2})$} \Comment{The first {\ttfamily Rank}-2 sub-blocks in the Communication block of phase $i$}
				\LState $I^{(j)}(t) = \mathcal{O}$
				\Comment{Play the most matched arm, which is input to this sub-routine}
				\EndFor
				\For {$\widehat{S}_i + K\max(0,\text{{\ttfamily Rank}-2}) < t \leq \widehat{S}_i + K(\max(0,\text{{\ttfamily Rank}-1}))$ }
				\Comment{The  {\ttfamily Rank}$-1^{th}$ sub-block in the Communication block of phase $i$}
				\LState $I^{(j)}(t) = ((t- (\widehat{S}_i + K\max(0,\text{{\ttfamily Rank}-2})) ) \mod K ) +1$ \Comment{Play arms in round robin}
				\If{ Collision Occurs} 
				\Comment{Collision in ({\ttfamily Rank}-1)th sub-block}
				\LState $\mathcal{C} \gets \mathcal{C} \cup \{ I^{(j)}(t)\}$ \Comment{Update set of arms to delete}
				\EndIf
				\EndFor
				\State \Return $\mathcal{C}$
\end{algorithmic}
\label{algo:algo_comm}
\end{algorithm*}
Observe that in the above algorithm, agent ranked $1$, will only play its best arm, in all rounds of the communication block. 

% \begin{algorithm*}
% 	\caption{DOMINATED-ARM-DETECTION ($\mathcal{O},i$, {\ttfamily Rank}) (at agent $j$)}
% \begin{algorithmic}
% \State \textbf{Input} $\mathcal{O} \in [K]$ - Arm to communicate, $i \in \mathbb{N}$ - the phase and {\ttfamily Rank} - the rank of the agent
%                 \LState $\mathcal{C} \gets \emptyset$
%                 \LState $\widehat{S}_i \gets S_i + 2^{i-1}+1$
% 				\For {$\widehat{S}_i \leq t < S_{i+1}$} \Comment{Communication block of phase $i$}
% 				\If{ $t - \widehat{S}_i < (\text{{\ttfamily Rank}}-1)*K$}
% 				\LState $I^{(j)}(t) = ((t-\widehat{S}_i) \mod K ) +1$ \Comment{Play arms in round robin}
% 				\If{ Collision Occurs \textbf{AND} $t-\widehat{S}_i > ($ {\ttfamily Rank}$-2)*K$} \\
% 				\Comment{Collision in ({\ttfamily Rank}-1)th sub-block}
% 				\LState $\mathcal{C} \gets \mathcal{C} \cup \{ I^{(j)}(t)\}$ \Comment{Update set of arms to delete}
% 				\EndIf
% 				\Else
% 				\LState $I^{(j)}(t) = \mathcal{O}$
% 				\Comment{Play the most matched arm, which is input to this sub-routine}
% 				\EndIf
% 				\EndFor
% 				\State \Return $\mathcal{C}$
% \end{algorithmic}
% \label{algo:algo_comm_old}
% \end{algorithm*}

\section{Analysis of the Algorithm and Proof of Theorem \ref{thm:main_thm_algo}}
\label{sec:analysis}

\subsection{Overall Proof Architecture}

The proof of Theorem \ref{thm:main_thm_algo} follows by plugging in the estimates from Corollary \ref{cor:reg_s_tilde}, into Corollary \ref{cor:reg_intermediate}. The subject matter in Appendix \ref{appendix:reg_decomp} is to prove Corollary \ref{cor:reg_intermediate} and the subject matter in Appendix \ref{sec:appendix_s_tilde} is to prove Corollary \ref{cor:reg_s_tilde}. All the notations and definitions needed for the proof are collected in Appendix \ref{appendix:proof_defns}.

\subsection{Notation and Definitions needed for the Proof}
\label{appendix:proof_defns}

In order to implement the proof, we specify certain notations and definitions.
For every $i \in \mathbb{N}$, denote by $S_i := N + (2^{i-1}-1) + (i-1)NK$, to be the first time slot in the regret minimization block of phase $i$. 
For any phase $i \in \mathbb{N}$, agent $j \in [N]$ and arm $k \in [K]$, denote by $\widetilde{N}_i^{(j)}[k]$ to be the number of times agent $j$ was matched to arm $k$ in phase $i$. 
Recall the notation that for all agents $j \in [N]$, its stable match partner arm was denoted as $k_*^{(j)} \in [K]$. Similarly the set of dominated arms for any agent $j \in \{2,\cdots,N\}$ we defined as $\mathcal{D}^*_j \coloneqq \{k_*^{(1)},\cdots,k_*^{(j-1)}\}$. Recall that we had set $\mathcal{A}_*^{(j)} \coloneqq [K]\setminus \mathcal{D}_*^{(j)}$. For any arm $k\in[K]$ and agent $j\in[N]$, denote by $\Delta^{(j)}_k := \mu_{jk^{(j)}_*} - \mu_{jk}$. 
\\

Our first definition is whether a given phase is good for a particular agent or not.
We call a phase $i \in \mathbb{N}$ \textbf{Good for Agent $j$} if 
\begin{enumerate}
    \item $\mathcal{A}_i^{(j)} = \mathcal{A}^*_j$, i.e., $\mathcal{A}^{(j)}_i = [K] \setminus \{k^{(1)}_*,\cdots,k^{(j-1)}_* \} $.
    \item The number of times each arm $k \in \mathcal{A}^{(j)}_* \setminus \{k_*^{(j)}\}$ is matched to agent $j$ in the regret minimization block of phase $i$ is less than or equal to $\frac{10 \alpha i}{(\Delta^{(j)}_k)^2}$. Note that by definition of $\Delta^{(j)}_k$ and the fact that arm-means are unique for an agent, for every agent $j \in [N]$ and arm $k \in \mathcal{A}^{(j)}_* \setminus \{k_*^{(j)}\}$, $\Delta^{(j)}_k > 0$.
    \item The arm that is most matched in the regret minimization block of phase $i$ is $k_*^{(j)}$.
    \end{enumerate}

We denote by the event $\chi_i^{(j)}$ to be the indicator random variable, i.e.,
\begin{align*}
    \chi_i^{(j)} = \mathbf{1}_{\text{Phase }i  \text{ is \textbf{Good} for Agent $j$}}.
\end{align*}

For every agent $j \in [N]$, denote by the random time $\tau^{(j)}$ to be the first phase index, such that all phases larger than $\tau^{(j)}$ is Good for agent $j$. Formally, 

\begin{align*}
    \mathcal{\tau}^{(j)} &:= \inf \left\{ i \in \mathbb{N} : \left(\prod_{l \geq i} \chi_l^{(j)} \right) = 1 \right\}, \\
    \widetilde{\tau}^{(j)} &:= \max (\tau^{(1)},\cdots,\tau^{(j)} ).
\end{align*}

Notice that after phase $i \geq \widetilde{\tau}^{(j)}$, for all agents $j' \leq j$, $\mathcal{A}_i^{(j')} = \mathcal{A}^{(j')}_*$. In other words, the set of active arms of all agents ranked $j$ and lower are `frozen' after phase $\widetilde{\tau}^{(j)}$ to the `correct' set of arms. 
\\

We now, describe certain set of events. 
For any agent $j \in [N]$ and arm $k \in [K] \setminus \{k_*^{(1)},\cdots,k_*^{(j)} \}$, denote by the event $\mathcal{E}_k^{(j)}$ as 
\begin{align}
    \mathcal{E}_k^{(j)} := \left\{ N_k^{(j)}(T) - N_k^{(j)}(S_{\widetilde{\tau}^{(j)}}) \geq \bigg\lceil\frac{4 \alpha \log(T)}{(\Delta_k^{(j)})^2}\bigg\rceil  \right\} \cap \{ \widetilde{\tau}^{(N)} < \infty \}.
    \label{eqn:defn_ejk}
\end{align}

Denote by the event $\mathcal{E}$ as the union, i.e., 
\begin{align}
    \mathcal{E} := \bigcap_{j=1}^N \bigcap_{k \in [K]\setminus \{k_*^{(1)},\cdots,k_*^{(j)} \}} \mathcal{E}_k^{(j)}.
    \label{eqn:defn_e}
\end{align}

Recall that we had defined $\Delta$ to be the smallest arm-gap, namely 
 \begin{align*}
     \Delta := \min_{j \in [N]} \min_{k \in \mathcal{A}^{(j)}_*\setminus \{ k_*^{(j)}\}} \Delta_k^{(j)},
 \end{align*}
 and that $i^* \in \mathbb{N}$ was defined as  
 \begin{equation}
     i^* := \min \left\{ i \in \mathbb{N}: \frac{20 NK \alpha i }{\Delta^2} \leq 2^{i-1} \right\}.
     \label{eqn:i_star_defn}
 \end{equation}
%  It is easy to verify the following corollary.
%  \begin{corollary}
%  \begin{align*}
%      2^{i^*} := O \left( \frac{NK}{\Delta^2} \log \left(\frac{NK}{\Delta^2} \right) \right).
%  \end{align*}
%  \label{cor:i^*_val}
%  \end{corollary}
 
\subsection{Technical Preliminaries}

In order to be precise in our calculations, we will need to explicitly specify a probability space. Let $\mathbb{Y} \coloneqq (Y_k^{(j)}(t))_{1 \leq k \leq K, 1 \leq j \leq N, t \geq 1}$, be a family of iid random variables, with each being defined uniformly in the interval $[0,1]$. The interpretation being that when agent $j$, gets matched with arm $k$, for the $t$th time, it receives a binary reward equal to $\mathbf{1}(Y_k^{(j)}(t) \leq \mu_k^{(j)})$. Thus, the dynamics of the algorithm can be constructed as a deterministic (measurable) function of the 
family of random variables $\mathbb{Y}$.

\subsection{Regret Decomposition}
\label{appendix:reg_decomp}

{\color{black}
\begin{lemma}
The regret of any agent $j \in [N]$, at time $T$ can be decomposed as
\begin{multline*}
       \mathbb{E} \left[R^{(j)}_T \right] \leq  \mathbb{E}[S_{\widetilde{\tau}^{(j)}}] + \underbrace{2(j-1)(\log_2(T)+1)}_{\text{Regret due to Communication}} + \\ \underbrace{\mathbb{E} \left[ \sum_{j^{'}=1}^{j-1} \sum_{k \in \mathcal{A}^{*}_j}(N_k^{(j^{'})}(T) - N_k^{(j^{'})}(S_{\widetilde{\tau}^{(j)}}))\mu_{jk^{(j)}_*} \right]}_{\text{Regret due to Collision}} + \underbrace{\mathbb{E} \left[ \sum_{k \in \mathcal{A}^{*}_j \setminus \{ k^{(j)}_* \}}(N_k^{(j)}(T) - N_k^{(j)}(S_{\widetilde{\tau}^{(j)}}))\Delta^{(j)}_k \right]}_{\text{Regret due to Sub-optimal arm pull}}.
\end{multline*}
\label{lem:reg_deecompose}
\end{lemma}
\begin{proof}
From the definition of regret, we have
\begin{align*}
\mathbb{E}[R_T(j)] &= \mathbb{E} \left[ \sum_{t=1}^T \mathbf{1}_{I^{(j)}(t) \neq k_{*}^{(j)} }(\mu_{jI^{(j)}(t)} - \mu_{jk^{(j)}_*})\right], \\
&= \mathbb{E} \left[\sum_{t=1}^{S_{\widetilde{\tau}^{(j)}}} \mathbf{1}_{I^{(j)}(t) \neq k_{*}^{(j)} }(\mu_{jI^{(j)}(t)} - \mu_{jk^{(j)}_*}) \right] + \mathbb{E} \left[\sum_{t=S_{\widetilde{\tau}^{(j)}}}^T \mathbf{1}_{I^{(j)}(t) \neq k_{*}^{(j)} +1}(\mu_{jI^{(j)}(t)} - \mu_{jk^{(j)}_*}) \right], \\
&\leq \mathbb{E}[S_{\widetilde{\tau}^{(j)}}] +  \mathbb{E} \left[\sum_{t=S_{\widetilde{\tau}^{(j)}}}^T \mathbf{1}_{I^{(j)}(t) \neq k_{*}^{(j)} }(\mu_{jI^{(j)}(t)} - \mu_{jk^{(j)}_*}) \right].
\end{align*}
The inequality follows from the assumption that, for all $j\in[N]$ and arm $k\in[K]$, $\mu_{jk}\in[0,1]$. Now, we decompose the second term as follows
\begin{multline*}
   \mathbb{E} \left[\sum_{t=S_{\widetilde{\tau}^{(j)}}}^T \mathbf{1}_{I^{(j)}(t) \neq k_{*}^{(j)} }(\mu_{jI^{(j)}(t)} - \mu_{jk^{(j)}_*}) \right] \leq  \underbrace{2(j-1)\mathbb{E}\left[ \sum_{i \geq \widetilde{\tau}^{(j)}}\mathbf{1}_{2^i \leq T}\right]}_{\text{Communication}} + \\ \underbrace{\mathbb{E} \left[ \sum_{j^{'}=1}^{j-1} \sum_{k \in \mathcal{A}^{*}_j}(N_k^{(j^{'})}(T) - N_k^{(j^{'})}(S_{\widetilde{\tau}^{(j)}}))\mu_{jk^{(j)}_*} \right]}_{\text{Collision}} + \underbrace{\mathbb{E} \left[ \sum_{k \in \mathcal{A}^{*}_j \setminus \{ k^{(j)}_* \}}(N_k^{(j)}(T) - N_k^{(j)}(S_{\widetilde{\tau}^{(j)}}))\Delta^{(j)}_k \right]}_{\text{Sub-optimal arm pull}}.
\end{multline*}
This in-equality follows from the following facts
\begin{itemize}
    \item During communication, agent ranked $j$ will face exactly $j-1$ collisions, in its round-robin arm search. Additionally, during each of the top-ranked $j-1$ agent's round robin arm-search, agent $j$ experiences one collision.
    \item Agent $j$ experiences a collision at an arm $k$, if and only if, exactly one agent ranked $1$ through $j-1$ is matched to arm $k$, at the same time. This then gives that the total upper bound on the number of collisions is the number of times agents $1$ through $j-1$, are matched to any arm in $\mathcal{A}^{(j)}_*$. This gives the regret due to collisions. 
    \item Finally, each sub-optimal arm match incurs regret, which is captured in the last term.
\end{itemize}
\end{proof}
It thus remains to bound the expected number of times, any agent $j^{'} \in [N]$, plays arm $k \in \mathcal{A}^{(j^{'})}_*$.

\begin{proposition}
For any agent $j \in [N]$, arm $k \in \mathcal{A}^{(j)}_* \setminus \{ k^{(j)}_*\}$, 
\begin{align*}
     \mathbb{E}[N_k^{(j)}(T) - N_k^{(j)}(S_{\widetilde{\tau}^{(j)}})] \leq \frac{32\alpha \log(T)}{(\Delta^{(j)}_k)^2} + 1 +\frac{8T^{-4\alpha}}{(\Delta^{(j)}_k)^2} + T^{2-4\alpha}.
\end{align*}
\label{prop:expectation}
\end{proposition}
\begin{proof}
For any agent $j \in [N]$, arm $k \in [K]$, and time $t \geq 2$, denote by $\text{UCB}^{(j)}_k(t) = \widehat{\mu}_{jk}(t-1)+ \sqrt{\frac{2\alpha \log(T)}{N^{(j)}_{k}(t-1)}}$, to be the UCB index of arm $k$, at time $t$, by agent $j$. Let $\varepsilon := \frac{\Delta^{(j)}_k}{2}$.
\begin{align*}
    \mathbb{E}[N_k^{(j)}(T) - N_k^{(j)}(S_{\widetilde{\tau}^{(j)}})] &= \mathbb{E} \left[ \sum_{t=S_{\widetilde{\tau}^{(j)}}+1}^{T} \mathbf{1}_{I^{(j)}(t) = k}\right], \\
    & \leq \mathbb{E} \left[  \sum_{t=S_{\widetilde{\tau}^{(j)}}+1}^{T} \mathbf{1}_{\text{UCB}^{(j)}_k(t) \geq \text{UCB}^{(j)}_{k^{*}_j}(t)}  \right], \\
    &\leq \mathbb{E} \left[ \sum_{t=S_{\widetilde{\tau}^{(j)}}+1}^{T} \left( \mathbf{1}_{\text{UCB}^{(j)}_k(t) \geq \mu_{jk^{(j)}_*}-\varepsilon, I^{(j)}(t)=k} + \mathbf{1}_{\text{UCB}^{(j)}_{k^{(j)}_*}(t) \leq \mu_{jk^{(j)}_*}-\varepsilon} \right)\right].
    \end{align*}
    The first inequality follows by the definition of the algorithm. The second inequality follows from a standard argument (c.f. Theorem $8.1$ of \cite{book1}). By linearity of expectation, we can rewrite the above as,
    \begin{multline*}
   \mathbb{E} \left[ \sum_{t=S_{\widetilde{\tau}^{(j)}}+1}^{T} \left( \mathbf{1}_{\text{UCB}^{(j)}_k(t) \geq \mu_{jk^{(j)}_*}-\varepsilon, I^{(j)}(t)=k} + \mathbf{1}_{\text{UCB}^{(j)}_{k^{(j)}_*}(t) \leq \mu_{jk^{(j)}_*}-\varepsilon} \right)\right] = \\ \sum_{t=1}^{T} \mathbb{P} \left[ \text{UCB}^{(j)}_k(t) \geq \mu_{jk^{(j)}_*}-\varepsilon, I^{(j)}(t)=k, t \geq S_{\widetilde{\tau}^{(j)}}+1\right] + \sum_{t=1}^{T} \mathbb{P} \left[ \text{UCB}^{(j)}_{k^{(j)}_*}(t) \leq \mu_{jk^{(j)}_*}-\varepsilon, t \geq S_{\widetilde{\tau}^{(j)}}+1 \right].
\end{multline*}
Each of the two terms can be computed in a standard fashion, as outlined in Chapter $8$ of \cite{book1}. We reproduce them here for completeness. For brevity, we are quite loose with the constants and have not optimized them. 

\begin{align*}
   \mathbb{P} \left[ \text{UCB}^{(j)}_k(t) \geq \mu_{jk^{(j)}_*}-\varepsilon, I^{(j)}(t)=k\right] &\leq    \mathbb{P} \left[ \bigcup_{s=1}^t \left\{N_k^{(j)}(t) = s, \widehat{\mu}_{jk}(t-1) + \sqrt{\frac{2\alpha \log_2(t)}{s}} \geq \mu_{jk^{(j)}_*}-\varepsilon \right\}  \right], \\
   &\leq    \mathbb{P} \left[ \bigcup_{s=1}^T \left\{N_k^{(j)}(t) = s, \widehat{\mu}_{jk}(t-1) + \sqrt{\frac{2\alpha \log(T)}{s}} \geq \mu_{jk^{(j)}_*}-\varepsilon \right\}  \right], \\
   &=  \mathbb{P} \left[ \bigcup_{s=1}^T \left\{ \widehat{\mu}_{jk,s} + \sqrt{\frac{2\alpha \log(T)}{s}} \geq \mu_{jk^{(j)}_*}-\varepsilon \right\}  \right],\\
   &\leq \sum_{s=1}^T \mathbb{P} \left[  \widehat{\mu}_{jk,s} + \sqrt{\frac{2\alpha \log(T)}{s}} \geq \mu_{jk^{(j)}_*}-\varepsilon \right], \\
   &= \sum_{s=1}^T \mathbb{P} \left[  \widehat{\mu}_{jk,s} + \sqrt{\frac{2\alpha \log(T)}{s}} \geq \mu_{jk} + \Delta^{(j)}_k -\varepsilon \right], \\
   &\stackrel{(a)}{=} \sum_{s=1}^T \mathbb{P} \left[  \widehat{\mu}_{jk,s} - \mu_{jk} \geq -\sqrt{\frac{2\alpha \log_2(T)}{s}}  + \frac{\Delta^{(j)}_k}{2}  \right], \\
   &\leq \frac{9\alpha \log(T)}{(\Delta^{(j)}_k)^2} + 1 + \sum_{s = \lfloor \frac{9\alpha \log(T)}{(\Delta^{(j)}_k)^2} \rfloor}^T \mathbb{P} \left[ \widehat{\mu}_{jk,s} - \mu_{jk} \geq -\sqrt{\frac{2\alpha \log(T)}{s}}  + 1 +\frac{\Delta^{(j)}_k}{2}  \right], \\
   &\stackrel{(b)}{\leq} \frac{9\alpha \log(T)}{(\Delta^{(j)}_k)^2} +1 + \sum_{s = \lceil \frac{9\alpha \log(T)}{(\Delta^{(j)}_k)^2} \rceil}^T \mathbb{P} \left[ \widehat{\mu}_{jk,s} - \mu_{jk} \geq  0.025{\Delta^{(j)}_k}  \right], \\
   &\stackrel{(c)}{\leq} \frac{9\alpha \log(T)}{(\Delta^{(j)}_k)^2} +1 + \sum_{s = \lceil \frac{9\alpha \log(T)}{(\Delta^{(j)}_k)^2} \rceil}^{\infty} \exp \left( -s (\Delta^{(j)}_k)^2 \frac{1}{1600}\right) , \\
  &\leq \frac{9\alpha \log(T)}{(\Delta^{(j)}_k)^2} + 1 +\frac{178T^{-4\alpha}}{(\Delta^{(j)}_k)^2}.
\end{align*}
Step $(a)$ follows from the definition that $\varepsilon = \frac{\Delta^{(j)}_k}{2}$. In Step $(b)$, we use that fact that for $s \geq \lceil \frac{9\alpha \log(T)}{(\Delta^{(j)}_k)^2} \rceil$, $\sqrt{\frac{2\alpha \log_2(T)}{s}} \leq \frac{\sqrt{2}\Delta^{(j)}_k}{3}$ and $0.5 - \frac{\sqrt{2}}{3} \geq 0.025$. In step $(c)$, we use Hoeffding's inequality. Similarly, we bound the other inequality as 
\begin{align*}
     \sum_{t=1}^{T} \mathbb{P} \left[ \text{UCB}^{(j)}_{k^{(j)}_*}(t) \leq \mu_{jk^{(j)}_*}-\varepsilon, t \geq S_{\widetilde{\tau}^{(j)}}+1 \right] &\leq     \sum_{t=1}^{T} \mathbb{P} \left[ \text{UCB}^{(j)}_{k^{(j)}_*}(t) \leq \mu_{jk^{(j)}_*}-\varepsilon\right], \\
     & = \sum_{t=1}^T \mathbb{P} \left[ \widehat{\mu}_{jk^{(j)}_*}(t) + \sqrt{\frac{2\alpha \log(T)}{N_{k^{(j)}_*}^{(j)}(t)}} \leq \mu_{jk^{(j)}_*} - \varepsilon \right], \\
     &\leq \sum_{t=1}^T \sum_{s=1}^t \mathbb{P} \left[ \widehat{\mu}_{jk^{(j)}_*, s}+ \sqrt{\frac{2\alpha \log(T)}{s}} \leq \mu_{jk^{(j)}_*} - \varepsilon \right],\\
     &\leq \sum_{t=1}^T \sum_{s=1}^t \exp \left( -2s \left( \sqrt{\frac{2\alpha \log(T)}{s}}   + \varepsilon \right)^2\right), \\
     &\leq T^{2-4\alpha}.
\end{align*}
\end{proof}

From the above two propositions, we obtain the following corollary
\begin{corollary}
The regret of any agent $j \in [N]$, at time $T$ can be bounded s 
\begin{multline*}
           \mathbb{E} \left[R^{(j)}_T \right] \leq  \mathbb{E}[S_{\widetilde{\tau}^{(j)}}] + \underbrace{2(j-1)(\log_2(T)+1)}_{\text{Regret due to Communication}} +  \underbrace{ \sum_{j^{'}=1}^{j-1} \sum_{k \in \mathcal{A}^{*}_j}\left( \frac{9\alpha \log(T)}{(\Delta^{(j^{'})}_k)^2} + 1 +\frac{178T^{-4\alpha}}{(\Delta^{(j^{'})}_k)^2} + T^{2-4\alpha}\right)\mu_{jk^{(j)}_*} }_{\text{Regret due to Collision}} \\ + \underbrace{ \sum_{k \in \mathcal{A}^{*}_j \setminus \{ k^{(j)}_* \}}\left( \frac{9\alpha \log(T)}{(\Delta^{(j)}_k)^2} + 1 +\frac{178T^{-4\alpha}}{(\Delta^{(j)}_k)^2} + T^{2-4\alpha}\right)\Delta^{(j)}_k }_{\text{Regret due to Sub-optimal arm pull}}.
\end{multline*}
\label{cor:reg_intermediate}
\end{corollary}
\begin{proof}
The Corollary follows from the following facts
\begin{itemize}
    \item For every $j^{'} < j$, $\widetilde{S}^{(j^{'})} \leq \widetilde{S}^{(j)}$ almost-surely.
    \item For every $j^{'} < j$, every arm in set $\mathcal{A}^{(j)}_*$ is sub-optimal.
    \item Plugging in the estimates from Proposition \ref{prop:expectation} in Lemma \ref{lem:reg_deecompose}.
\end{itemize}
\end{proof}

}

Thus, it remains to bound $\mathbb{E}[S_{\widetilde{\tau}^{(j)}}]$, which is the subject of the next section.

\subsection{Bound on Mean and Exponential Moment of $\widetilde{\tau}^{(j)}$}
\label{sec:appendix_s_tilde}

%In order to complete the regret guarantee, it suffices to bound $\mathbb{E}[S_{\widetilde{\tau}^{(j)}}]$ and $\mathbb{E}[2^{S_{\widetilde{\tau}^{(j)}}}]$, for all $j in [N]$. 
Since for all $i \in \mathbb{N}$, $S_i = N + (2^{i-1}-1)+(i-1)NK$, it suffices to bound the exponential moment $\mathbb{E}[2^{\widetilde{\tau}^{(j)}}]$ and the mean $\mathbb{E}[\widetilde{\tau}^{(j)}]$ to complete the regret guarantee.
 In order to do so, we first start by analyzing the probability that a phase is bad for an agent and then use that to bound the exponential moment of $\tau^{(j)}$.
 We shall now bound the probability that a phase is bad for a particular agent.
 
 \begin{lemma}
 For any phase $i > i^*$, any agent $j$ and arm $k \in \mathcal{A}^{(j)}_*\setminus \{ k^{(j)}_*\}$, we have 
 \begin{align*}
     \mathbb{P}[\chi_i^{(j)} = 0,i \geq \widetilde{\tau}^{(j-1)}] \leq 2N^2 \left( \frac{2}{e^5}\right)^i.
 \end{align*}
 \label{lem:chi_tail_bound}
 \end{lemma}
 \begin{proof}
 The proof follows from the basic properties of the UCB algorithm, which we can bound as follows. 
 Recall the notation that for any agent $j \in [N]$, phase $i \in \mathbb{N}$ and arm $k \in [K]$, the quantity $N_k^{(j)}[i]$ denotes the number of times agent $j$ was matched to arm $k$ in the regret minimization blocks upto and including phase $i$.
 For any phase $i \geq i^*$, agent $j$ and arm $k \in \mathcal{A}^{(j)}_*\setminus \{k^{(j)}_*\}$, we have 
 \begin{multline*}
     \mathbb{P}\left[N_k^{(j)}[i] - N_k^{(j)}[i-1] > \bigg\lceil \frac{10 \alpha i}{(\Delta_k^{(j)})^2} \bigg\rceil, i \geq \widetilde{\tau}^{(j-1)} \right] \leq \\ \mathbb{P} \left[ \bigcup_{t \geq S_{i-1} + \lceil \frac{10 \alpha i}{(\Delta_k^{(j)})^2}\rceil}^{S_i}  N_k^{(j)}(t) = \bigg\lceil\frac{10 \alpha i}{(\Delta_k^{(j)})^2}\bigg\rceil + N_k^{(j)}[i-1],I^{(j)}(t) = k,i \geq \widetilde{\tau}^{(j-1)}\right].
 \end{multline*}
 Since $N_k^{(j)}[i-1] \geq 0$, the above can be simplified to 
  \begin{multline}
     \mathbb{P}\left[N_k^{(j)}[i] - N_k^{(j)}[i-1] > \bigg\lceil \frac{10 \alpha i}{(\Delta_k^{(j)})^2} \bigg\rceil, i \geq \widetilde{\tau}^{(j-1)} \right] \leq \\ \mathbb{P} \left[ \bigcup_{t \geq S_{i-1} + \lceil \frac{10 \alpha i}{(\Delta_k^{(j)})^2}\rceil}^{S_i}  N_k^{(j)}(t) \geq \bigg\lceil\frac{10 \alpha i}{(\Delta_k^{(j)})^2}\bigg\rceil ,i \geq \widetilde{\tau}^{(j-1)}\right].
     \label{eqn:inter1}
 \end{multline}

 Now, by applying an union bound to the RHS, we obtain from the preceding display that 
 \begin{multline}
    \mathbb{P} \left[ \bigcup_{t \geq S_{i-1} + \lceil \frac{10 \alpha i}{(\Delta_k^{(j)})^2}\rceil}^{S_i}  N_k^{(j)}(t) \geq \bigg\lceil\frac{10 \alpha i}{(\Delta_k^{(j)})^2}\bigg\rceil,I^{(j)}(t) = k,i \geq \widetilde{\tau}^{(j-1)}\right] \\ \leq   \sum_{t=S_{i-1}}^{S_i} \mathbb{P} \left[N_k^{(j)}(t) \geq \bigg\lceil\frac{10 \alpha i}{(\Delta_k^{(j)})^2}\bigg\rceil,I^{(j)}(t) = k, i \geq \widetilde{\tau}^{(j-1)}
     \right].
     \label{eqn:inter2}
 \end{multline}
 The classical large-deviation estimate for UCB from \cite{auer-ucb} gives that 
 \begin{align}
     \mathbb{P} \left[N_k^{(j)}(t) \geq \bigg\lceil\frac{10 \alpha i}{(\Delta_k^{(j)})^2}\bigg\rceil,I^{(j)}(t) = k, i \geq \widetilde{\tau}^{(j-1)}
     \right] \leq 2e^{-5i},
     \label{eqn:inter3}
 \end{align}
 for all times $t \in \{S_{i-1},\cdots,S_{i}\}$. In words, this estimate gives that UCB will play a sub-optimal arm with very low probability, on the event that the sub-optimal arm has been played sufficiently many times.
 We can use that estimate, since on the event that $i \geq \widetilde{\tau}^{(j-1)}$, we have that the set of active arms of agent $j$ in phase $i$, denoted by $\mathcal{A}^{(j)}_i = \mathcal{A}^{(j)}_*$. Thus, arm $k^{(j)}_*$ is the best arm for agent $j$ in phase $i$.
 Now, combining Equations (\ref{eqn:inter1}),(\ref{eqn:inter2}),(\ref{eqn:inter3}), we get
 \begin{align}
 \mathbb{P}\left[N_k^{(j)}[i] - N_k^{(j)}[i-1] > \bigg\lceil \frac{10 \alpha i}{(\Delta_k^{(j)})^2} \bigg\rceil, i \geq \widetilde{\tau}^{(j-1)} \right] 
     &\stackrel{}{\leq} \sum_{t=S_{i-1}}^{S_i} 2e^{-5i} \leq 2 \left( \frac{2}{e^5} \right)^i \label{eqn:chi_tail}.
 \end{align}
%  In step $(a)$, we use the fact that $\alpha > 2$ and the classical estimate from UCB algorithm from \cite{auer-ucb}. We can use that estimate, since on the event that $i \geq \widetilde{\tau}^{(j-1)}$, we have that the set of active arms of agent $j$ in phase $i$, denoted by $\mathcal{A}^{(j)}_i = \mathcal{A}^{(j)}_*$. Thus, arm $J^*_j$ is the best arm for agent $j$ in phase $i$.
To conclude the proof, notice the following fact.
 
 \begin{proposition}
 For every $i \geq i^*$,
 \begin{align*}
    \{ \chi_i^{(j)} = 0, i\geq \widetilde{\tau}^{(j-1)} \} \subseteq  \bigcup_{j'=1}^{j} \bigcup_{k \in \mathcal{A}^{(j')}_*\setminus \{k^{(j')}_* \}} \left \{[N_k^{(j^{'})}[i] - N_k^{(j^{'})}[i-1] \geq \bigg\lceil \frac{10 \alpha i}{(\Delta_k^{(j^{'})})^2} \bigg\rceil , i \geq \widetilde{\tau}^{(j-1)} \right \}  ,
 \end{align*}
 where $i^*$ is defined in Equation (\ref{eqn:i_star_defn}).
 \label{prop:chi_events}
 \end{proposition}
 \begin{proof}
 It suffices to establish that 
 \begin{align*}
 \bigcap_{j'=1}^{j} \bigcap_{k \in \mathcal{A}_*^{(j')}\setminus \{k^{(j')}_* \}} \left \{[N_k^{(j^{'})}[i] - N_k^{(j^{'})}[i-1] \leq \bigg\lceil \frac{10 \alpha i}{(\Delta_k^{(j^{'})})^2} \bigg\rceil , i \geq \widetilde{\tau}^{(j-1)} \right \}  \subseteq        \{ \chi_i^{(j)} = 1, i\geq \widetilde{\tau}^{(j-1)} \} .
 \end{align*}
 This follows as, the phase $i^*$ is such that $\frac{10 \alpha i}{\Delta^2}NK < 2^{i-1}$. Suppose all events on the LHS hold. Then, agent $j$ is matched at-most $\frac{10 \alpha i}{(\Delta^{(j)}_k)^2}$ times to any arm $\mathcal{A}^{(j)}_*\setminus \{k^{(j)}_*\}$ (the sub-optimal arms).
 This in turn is upper bounded by $\frac{10 \alpha i}{\Delta^2}$. 
 Since there are $K-j$ arms in the set $\mathcal{A}^{(j)}_*\setminus \{k^{(j)}_*\}$, the total number of sub-optimal arm pulls is at-most $(K-j)\frac{10 \alpha i}{\Delta^2}$.
 \\

 In order to bound the total number of collisions agent $j$ will face, we make use of two simple observations. First is that, under the events on the LHS, every agent $j^{'} < j$ will match no more than $\frac{10 \alpha i}{(\Delta^{(j^{'})}_k)2}$ times with arm $k$, for all $k \in \mathcal{A}_*^{(j)}$. Second, is that if agent $j$ faces a collision at an arm $k \in \mathcal{A}_*^{(j)}$, then it must be the case that exactly one agent ranked $1$ through to $j-1$ must have been matched to arm $k$ in the same slot. 
 These two observations give that the total number of times agent $j$ will face a collision at an arm $k \in \mathcal{A}_*^{(j)}$ is at-most the sum of times arm $k$ is matched to agents ranked $1$ through to $j-1$, which in turn is upper bounded by $(j-1)\frac{10 \alpha i}{\Delta^2}$. Since therer are exactly $K-j+1$ arms in $\mathcal{A}_*^{(j)}$, the total number of collisions incurred by agent $j$ in phase $i$, when the events on the LHS hold is upper bounded by $(K-j+1)(j-1)\frac{10 \alpha i}{\Delta^2}$, which in turn is upper bounded by $NK\frac{10 \alpha i}{\Delta^2}$.
 \\
 
 However, as there is at-least $20NK\frac{ \alpha i}{\Delta^2}$ time slots in phase $i$, the preceding argument yields that agent $j$ must be matched to arm $k_*^{(j)}$ at-least $20NK\frac{ \alpha i}{\Delta^2} - NK\frac{10 \alpha i}{\Delta^2} - K\frac{10\alpha i}{\Delta^2} \geq NK \frac{10 \alpha i}{\Delta^2}$ times. Thus, agent $j$ is matched to arm $k^{(j)}_*$ the most number of times in phase $i$, i.e., $\chi_i^{(j)} = 1$.
 
 %Thus, the total number of collisions experienced by agent $j$ in phase $i$ is at-most $NK \frac{10\alpha i}{\Delta^2}$. However, as the total number of time slots in phase $i$ is at-least $20NK\frac{ \alpha i}{\Delta^2}$, it follows that agent $j$ must be matched to arm $k^{(j)}_*$ at-least $NK \frac{10 \alpha i}{\Delta^2}$ times. Thus, agent $j$ is matched to arm $k^{(j)}_*$ the most number of times in phase $i$. 
 \end{proof}
 Thus, from Proposition \ref{prop:chi_events}, and applying an union bound using Equation (\ref{eqn:chi_tail}), we get
 \begin{align*}
     \mathbb{P}[\chi_i^{(j)} = 0, i \geq \widetilde{\tau}^{(j-1)}] &\leq 2jK \left( \frac{2}{e}\right)^i,\\
     &\leq 2NK \left( \frac{2}{e}\right)^i.
 \end{align*}
 \end{proof}
 
 We use this to now compute the mean of $\widetilde{\tau}^{(j)}$.
 
 \begin{proposition}
 For every $j \in [N]$, we have
 \begin{align*}
     \mathbb{E}[\widetilde{\tau}^{(j)}] \leq j(i^* + 3NK),
 \end{align*}
 where $i^*$ is defined in Equation (\ref{eqn:i_star_defn}).
 \label{prop:mean}
 \end{proposition}
 
 \begin{proof}
 \begin{align*}
     \mathbb{E}[\widetilde{\tau}^{(j)}] &= \sum_{x \geq 1}\mathbb{P}[\widetilde{\tau}^{(j)} \geq x],\\
     &= \sum_{x \geq 1}\mathbb{P}[\widetilde{\tau}^{(j)} \geq x, \widetilde{\tau}^{(j-1)} > x] + \sum_{x \geq 1}\mathbb{P}[\widetilde{\tau}^{(j)} \geq x, \widetilde{\tau}^{(j-1)} \leq x],\\
     &= \sum_{x \geq 1}\mathbb{P}[\widetilde{\tau}^{(j-1)} > x] + \sum_{x \geq 1}\mathbb{P}[\widetilde{\tau}^{(j)} \geq x, \widetilde{\tau}^{(j-1)} \leq x],\\
     &\leq \mathbb{E}[\widetilde{\tau}^{(j-1)}] + i^* + \sum_{x \geq i^*}\mathbb{P}[\widetilde{\tau}^{(j)} \geq x, \widetilde{\tau}^{(j-1)} \leq x],\\
     &\stackrel{(a)}{\leq} \mathbb{E}[\widetilde{\tau}^{(j-1)}] + i^* + \sum_{x \geq i^*} 2N^2 \left( \frac{2}{e^5} \right)^x,\\
     &\leq \mathbb{E}[\widetilde{\tau}^{(j-1)}] + i^* + 3NK,\\
     &\stackrel{(b)}{\leq} j(i^* + 3NK).
 \end{align*}
 Step $(a)$ follows Lemma \ref{lem:chi_tail_bound} and step $(b)$ follows from the fact that $\widetilde{\tau}^{(0)} = 0$ almost-surely.
 \end{proof}
 Similarly, we can compute the exponential moment of $\widetilde{\tau}^{(j)}$.
 \begin{proposition}
  For every $j \in [N]$, we have
 \begin{align*}
     \mathbb{E}[2^{\widetilde{\tau}^{(j)}}] \leq 1 + j(2^{i^*} + 4NK),
 \end{align*}
  where $i^*$ is defined in Equation (\ref{eqn:i_star_defn}).
   \label{prop:exponential_bound}
 \end{proposition}

 \begin{proof}
 \begin{align*}
     \mathbb{E}[2^{\widetilde{\tau}^{(j)}}] &= \sum_{x \geq 1}\mathbb{P}[2^{\widetilde{\tau}^{(j)}} \geq x],\\
     &= \sum_{x \geq 1}\mathbb{P}[ \widetilde{\tau}^{(j)} \geq \log_2(x) ],\\
     &= \sum_{x \geq 1}\mathbb{P}[ \widetilde{\tau}^{(j)} \geq \log_2(x), \widetilde{\tau}^{(j-1)} > \log_2(x) ] + \sum_{x \geq 1}\mathbb{P}[ \widetilde{\tau}^{(j)} \geq \log_2(x), \widetilde{\tau}^{(j-1)} \leq \log_2(x)  ],\\
         &= \sum_{x \geq 1}\mathbb{P}[  \widetilde{\tau}^{(j-1)} > \log_2(x) ] + \sum_{x \geq 1}\mathbb{P}[ \widetilde{\tau}^{(j)} \geq \log_2(x), \widetilde{\tau}^{(j-1)} \leq \log_2(x)  ],\\
         &= \mathbb{E}[2^{\widetilde{\tau}^{(j-1)}}] + \sum_{x \geq 1}\mathbb{P}[ \widetilde{\tau}^{(j)} \geq \log_2(x), \widetilde{\tau}^{(j-1)} \leq \log_2(x)  ],\\
         &\stackrel{(a)}{\leq} \mathbb{E}[2^{\widetilde{\tau}^{(j-1)}}] + 2^{i^*} + \sum_{x \geq 2^{i^*}} 2NK \left( \frac{2}{e^5}\right)^{\log_2x},\\
         &= \mathbb{E}[2^{\widetilde{\tau}^{(j-1)}}] + 2^{i^*} + \sum_{x \geq 2^{i^*}} 2NK x^{1 - \frac{5}{\ln(2)}},\\
         &\leq \mathbb{E}[2^{\widetilde{\tau}^{(j-1)}}] + 2^{i^*} + 4NK,\\
         &\stackrel{(b)}{\leq} 1+j(2^{i^*}+4NK).
 \end{align*}
 Step $(a)$ follows from Lemma \ref{lem:chi_tail_bound} and step $(b)$ follows from the fact that $\widetilde{\tau}^{(0)} = 0$ almost-surely. 
 \end{proof}
 
 \begin{corollary}
 For every $j \in [N]$, we have 
 \begin{align*}
     \mathbb{E}[S_{\widetilde{\tau}^{(j)}}] &\leq N + j(2^{i^*} + 4NK) + jNK(i^* + 3NK),\\
     &= N + j(2^{i^*}+i^* + 4(NK)^2 + 3NK),
 \end{align*}
  where $i^*$ is defined in Equation (\ref{eqn:i_star_defn}).
  \label{cor:reg_s_tilde}
 \end{corollary}
 \begin{proof}
 We know that for any $i \in \mathbb{N}$, $S_i := N + (2^{i-1}-1)+(i-1)NK$. The result then follows from Propositions \ref{prop:mean} and \ref{prop:exponential_bound}.
 \end{proof}

\section{Incentive Compatibility - Proof of Proposition \ref{prop:incentive}}
\label{appendix:proof-incentive}

We restate Proposition \ref{prop:incentive} for the reader's convenience.
\begin{proposition}
The \name\; algorithm profile is $\varepsilon : (\varepsilon_j)_{j=1}^N$ stable  where, for all $j\in[N]$, $\varepsilon_j = \sum_{l=1}^{j-1}\mathbf{1}_{(\mu_{jl} > \mu_{jk^*_j})}\frac{\mu_{jl}}{\mu^{(l)}}\mathbb{E}[R_T^{(l)}] + \mathbb{E}[R_T^{(j)}]$, where for all $j' \in [N]$, $\mathbb{E}[R_T^{(j')}]$ is given in Equation (\ref{eqn:main_thm}).
%\label{prop:incentive}
\end{proposition}
This proposition gives that for agent ranked $j$, $\varepsilon_j = O \left( j \frac{\mu^{(j)}_{\text{max}}}{\mu^{(j)}_{\text{min}}} \mathbb{E}[R_T^{((j))}] \right) $, where $\mu^{(j)}_{\text{max}}$ ($\mu^{(j)}_{\text{min}}$) is the maximum (minimum) arm-mean for agent $j$. 
%The proof is in Section \ref{sec:proof-incentive}.
This establishes that \name\; is approximately incentive compatible, namely, even if an agent deviates from the \name\; algorithm, the possible improvement in reward is $O(\log(T))$.

{
\begin{proof}[Proof of Proposition~\ref{prop:incentive}]
We bound the equilibrium property of \name; as follows. 
Observe that agent ranked $j$ will only collide with agents ranked $1$ through $j-1$. 
Now, if all agents $1$ through to $j-1$ are all playing arms $k_*^{(1)},\cdots k_*^{(j-1)}$ respectively (their individual best arms), then the best arm (by definition) for agent $j$ to play will be arm $k_*^{(j)}$. 
On the other hand, when any agent $j' \leq j-1$ does not play arm $k_*^{(j')}$, the maximum expected reward collected by agent $j$ can be at-most $\max(\mu_{j k_*^{(j')}},\mu_{jk_*^{(j)}} )$. 
Under the \name; strategy profile, the expected number of times any agent $j \in [N]$, plays an arm in the set $[K]\setminus \{k_*^{(1)},\cdots,k_*^{(j)}\}$ is at-most $\frac{1}{\mu^{(j)}}\mathbb{E}[R_T^{(j)}]$, where $\mu^{(j)}:= \min_{k \in [K] } \mu_{jk}$ is the smallest arm-gap. Notice that for all agents $j$, $\mu^{(j)} > 0$, by model assumptions.
This then gives us the following decomposition
\begin{align}
    \sup_{s'} \mathbb{E}[\text{Rew}^{(j)}_T(s_{-j},s')] \leq  \sum_{l=1}^{j-1}\mathbf{1}_{\mu_{jl > \mu_{jk_*^{(j)}}}}\frac{\mu_{jl}}{\mu^{(l)}}\mathbb{E}[R_T^{(l)}] + \mu_{jk_*^{(j)}}T,
    \label{eqn:ub1}
\end{align}
where $\mathbb{E}[R_T^{(l)}]$ is given in Theorem \ref{thm:main_thm_algo}.
Similarly, from the definition of regret, we have
\begin{align}
    \mathbb{E}[\text{Rew}_T^{(j)}(s)] \geq \mu_{jk_*^{(j)}}T -  \mathbb{E}[R_T^{(j)}],
    \label{eqn:lb1}
\end{align}
where $\mathbb{E}[R_T^{(j)}]$ is given in Theorem \ref{thm:main_thm_algo}. Thus, from Equations (\ref{eqn:ub1}) and (\ref{eqn:lb1}), we get that
\begin{align*}
   \sup_{s'} \left( \mathbb{E}[\text{Rew}^{(j)}_T(s_{-j},s')]   -  \mathbb{E}[\text{Rew}_T^{(j)}(s)] \right) \leq \sum_{l=1}^{j-1}\mathbf{1}_{\mu_{jl > \mu_{jk_*^{(j)}}}}\frac{\mu_{jl}}{\mu^{(l)}}\mathbb{E}[R_T^{(l)}] + \mathbb{E}[R_T^{(j)}].
\end{align*}
% \textbf{Stability} - 
% We now consider the stability property. Due to the uniform valuations, an agent $j$ can only reduce the regret of agents ranked $j+1$ or lower. Let us fix agents $j$ and $j^{'} > j$. Thus, if agent $j$ needs to cause $l \in \mathbb{N}$ extra collisions to agent $j^{'}$, the following must be satisfied - {\em (i)} under strategy profile $s$, agent $j^{'}$ must be matched to some arm and {\em (ii)} under strategy profile $(s_{-j},s^{'})$, agent $j^{'}$ must suffer a collision (from possibly agent $j$).

% {\color{red}
% Now, to prove the stability, we make the following observation. If an agent $j$ deviates from \name\; algorithm, then it can only cause additional regret to any other agent $l > j$. This is because, agent $j$ cannot cause a collision (and thus loss of reward) to any other agent ranked higher than itself. Now, observe that if agent $j$ causes $C$ extra collisions to agent $l > j$, then agent $l$ looses reward at-most $\widehat{\mu}^{(j^{'})}C$, where for any agent $j \in [N]$, $\widehat{\mu}^{(j)} = \max_{k \in[N]} \mu_{jk}$. In order to cause $C$ collisions, the minimum loss of reward for agent $j$ is $\Delta^{(j)}C$. }
\end{proof}

}
\section{Proof of Regret Lower Bound}
\label{appendix-lower-bound}

We will use the following notations throughout the proof of the lower bound.
\begin{enumerate}
\item Can assume without loss of generality (W.l.o.g.) that the rank of any agent $i \in [N]$ is $i$.

\item Any agent related symbol is a superscript. Arm related is a sub-script. Thus, for any time $t$, the number of times arm $k \in [K]$ is played by agent $j$ is $N^{(j)}_k(t)$. The number of time the agent $j$ is blocked up to time $t$ is given as $C^{(j)}(t)$. 

\item Distribution of agent $i \in [N]$ and arm $k \in [K]$ is given by $\nu_{jk}$, which has mean $\mu_{jk}$. W.l.o.g. let us assume $\max_k \mu_{jk} > 0$.

\item The stable match partner of any agent $j \in [N]$ is given by $k_*^{(j)} \in [K]$. 
The set of dominated arms for the agent $j$ is given as $\mathcal{D}^{(j)}_* = \{k_*^{(j)}: 1\leq j'\leq j-1\}$, the set of non-dominated arms is given as $\mathcal{A}^{(j)}_* = [K]\setminus \mathcal{D}^{(j)}_*$.
\item For any agent $j \in [N]$, arm $k \in [K]$, $\Delta^{(j)}_k := \mu_{j k_*^{(i)}} - \mu_{jk}$, the arm-gap. This can be negative. 
\end{enumerate}

\subsection{Divergence Decomposition}
We need to setup a few notations for the proof of divergence decomposition lemma. The proof generalizes the framework in Chapter 15 of \cite{book1}  for the multi-agent framework.\footnote{See, \cite{multi_revisited} for a related approach for regret lower bound proof in the colliding bandit models~\cite{cognitive_bandit}.}
\\

\textbf{Canonical multi-agent bandit model:}  We now define the  ($N$-agent, $K$-arm, T-horizon) bandit models. The canonical bandit model  ($N$-agent, $K$-arm, T-horizon)  lies in a measurable space $\{\Omega, \mathcal{F}\}$.  Let $K^{(j)}(t)$ denotes the arm chosen by the $j$-th agent on time $t$, and $X^{(j)}(t)$ denotes  the rejection or reward obtained from that arm for agent $j$ in round $t$. We denote the rejection by the symbol $\emptyset$. Therefore, $K^{(j)}(t) \in [K]$, and $X^{(j)}(t) \in [0,1] \cup \{\emptyset\}$ for all $j \in [N]$ and $t \in [T]$.  Also, $K^{(j)}(t)$, and $X^{(j)}(t)$ for all $j \in [N]$ and $t \in [T]$ are measurable with respect to $\mathcal{F}$. Let $H(t)= (K^{(j)}(t'), X^{(j)}(t') \forall j \in [N], \forall t'\leq t)$ be the random variable representing the history of actions taken and rewards seen up to and including round $t$. We have $H(t) \in \mathcal{H}(t)\equiv\left([K]^N\times \left ([0,1] \cup \{\emptyset\}\right)^N\right)^{t}$. We may set $ \Omega \equiv\mathcal{H}(T)$ and the sigma algebra generated by the history as $\mathcal{F} \equiv\sigma(H(T))$.
\\

\textbf{Environment:} The bandit environment is specified by $\nu = (\nu_{jk} , \forall j \in [N], k \in [K])$ where $\nu_{jk}$ is the distribution of rewards obtained when arm $k$ is matched to agent $j$ in this environment. 
\\

\textbf{Policy:} A policy is a sequence of distribution of possible request to the arms from the agents (which can assimilate any coordination among the agents) conditioned on the past events. More formally, the policy $\boldsymbol{\pi} = \{\boldsymbol{\pi}_t(\cdot): t\in [T]\}$ where $\boldsymbol{\pi}_t(\cdot)\equiv\{\pi_t(k,j|\cdot): \forall k \in [K], j \in [N]\}$ is the function that maps the history upto time $t-1$ to the action $K^{(j)}(t), \forall j\in [N]$. Further,  $\pi_t(k,j|\cdot): \mathcal{H}(t-1) \to [0,1]$ denotes the probability, as a  function of history $H(t-1)$ of agent $j$ playing arm $k$. 
\\

\textbf{Probability Measure:} Each environment $\boldsymbol{\nu}$ and policy $\pi$ jointly induces a probability distribution over the measurable space $\{\Omega, \mathcal{F}\}$ denoted by $\mathbb{P}_{\nu, \pi}$.  Let $\mathbb{E}_{\nu, \pi}$ denote the expectation induced. The density of a particular history up to time $T$, under an environment $\boldsymbol{\nu}$ and a policy $\pi$, can be defined as 
\begin{align*}
&d\mathbb{P}_{\nu, \pi}\left(\mathbf{k}(t), \mathbf{x}(t): t \in [T]\right) \\
&= \prod_{t=1}^{T}\pi_t(\mathbf{k}(t)| h(t-1)) p_{\nu}(\mathbf{x}(t)|\mathbf{k}(t)) d\lambda(\mathbf{x}(t); \boldsymbol{\nu}) d\rho(\mathbf{k}(t)).  
\end{align*}
Here, $\lambda(\mathbf{x}; \boldsymbol{\nu})=\prod_{j=1}^{N}\lambda_j(x^{(j)})$ is the dominating measure over the rewards with $\lambda_j(x^{(j)}) = \delta_{\emptyset} + \sum_k \nu_{jk}$.\footnotemark Also,  $\rho(\mathbf{k})$ is the counting measure on the collective action of the agents. 
\\
 
\footnotetext{Here $\delta_{\emptyset}$ is the dirac measure on $\emptyset$ denoting the rejection event. For multiple a pair environments we can define a dominating measure as 
$\lambda(\mathbf{x}; \boldsymbol{\nu}_1,\boldsymbol{\nu}_2) = \sum_{i=1,2}\lambda(\mathbf{x}; \boldsymbol{\nu}_i)$. This is used in the proof of Lemma~\ref{lemm:divergence}.}

\begin{lemma}[Divergence Decomposition]\label{lemm:divergence}
For two bandit instances $\boldsymbol{\nu} = \{\nu_{jk} : j\in [N], k\in [K]\}$, and $\boldsymbol{\nu}' = \{\nu_{jk} : j\in [N], k\in [K]\}$, and any admissible policy $\pi$ the following divergence decomposition is true 
$$
D(\mathbb{P}_{\nu, \pi}, \mathbb{P}_{\nu', \pi}) = \sum_{j = 1}^{N}\sum_{k = 1}^{K} \mathbb{E}_{\nu, \pi}[N^{(j)}_k(T)] D(\nu_{jk}, \nu'_{jk}).
$$
\end{lemma}
\begin{proof}
The divergence between two measures, which correspond to two different environments under a policy $\pi$, $\mathbb{P}_{\nu, \pi}$ and $\mathbb{P}_{\nu', \pi}$ can be expressed as 
\begin{align*}
&D(\mathbb{P}_{\nu, \pi}, \mathbb{P}_{\nu', \pi})\\
&= \mathbb{E}_{\nu, \pi}\left[\sum_{t=1}^{T}\log\left(\frac{d \mathbb{P}_{\nu, \pi}}{d \mathbb{P}_{\nu', \pi}}\right)\right]\\
&\stackrel{(i)}{=} \mathbb{E}_{\nu, \pi}\left[\log\left(\frac{p_{\nu}(\mathbf{x}(t)|  \mathbf{k}(t))}{p_{\nu'}(\mathbf{x}(t)| \mathbf{k}(t))}\right)\right]\\
& \stackrel{(ii)}{=}\mathbb{E}_{\nu, \pi}\left[\sum_{t=1}^{T}\log\left(\frac{ \prod_{j: x^{(j)}(t) \neq \emptyset} p_{\nu}(x^{(j)}(t)| \mathbf{k}(t))}{ \prod_{j: x^{(j)}(t) \neq \emptyset} p_{\nu'}(x^{(j)}(t)|\mathbf{k}(t))}\right)\right]\\
&=\mathbb{E}_{\nu, \pi}\left[\sum_{t=1}^{T}  \sum_{j: x^{(j)}(t) \neq \emptyset} \mathbb{E}_{\nu} \left[\log\left(\frac{p_{\nu}(x^{(j)}(t)| \mathbf{k}(t))}{ p_{\nu'}(x^{(j)}(t)|\mathbf{k}(t))}\right)\bigg\lvert\mathbf{k}(t)\right]\right]\\
&\stackrel{(iii)}{=}\mathbb{E}_{\nu, \pi}\left[\sum_{t=1}^{T}  \sum_{j: x^{(j)}(t) \neq \emptyset} D\left(\nu_{jk^{(j)}(t)}, \nu'_{jk^{(j)}(t)}\right)\right]\\
&=\sum_{j=1}^{N}\sum_{k=1}^{K} \mathbb{E}_{\nu, \pi}
\left[\sum_{t=1}^{T} \mathbf{1}_{(k^{(j)}(t) = k, x^{(j)}(t) \neq \emptyset)}  D\left(\nu_{jk}, \nu'_{jk}\right) \right]\\
&\stackrel{(iv)}{=}\sum_{j=1}^{N}\sum_{k=1}^{K}  \mathbb{E}_{\nu, \pi}
\left[N^{(j)}_k(T)\right] D\left(\nu_{jk}, \nu'_{jk}\right).
\end{align*}
In the above series  equation (i) is true because the density of the policy cancels out for the two different environments.
Equation (ii) holds because if for some set of actions $\mathbf{k}(t)$ agent $j$ observes $x^{(j)}(t) =\emptyset$ that indicates agent $j$ is rejected on that round. This is independent of the environment. In particular, we have  $p_{\nu}(x^{(j)}(t)|\mathbf{k}(t)) = p_{\nu'}(x^{(j)}(t)|\mathbf{k}(t))$ if $x^{(j)}(t) =\emptyset$ for any $\mathbf{k}(t)$. In deriving equation (iii) we make use of the definition of divergence. Equation (iv) uses the definition of $N^{(j)}_k(t)$ the total number of times agent $j$ successfully plays arm $k$ up to time $T$. 
\end{proof}

\subsection{Proof of Regret Decomposition (Lemma~\ref{lemm:regret})}
\begin{proof}
We fix any agent  $j \in [N]$ for the rest of the proof. We have the expected regret for the agent $j$, under a policy $\pi$ and any bandit instance $\boldsymbol{\nu}$  as 
$$
R^{(j)}_T(\boldsymbol{\nu}, \pi) = \sum_{k=1}^{K}\Delta^{(j)}_{k}\mathbb{E}_{\nu, \pi}[N^{(j)}_k(T)] + 
\sum_{k=1}^{K}\mu_{jk_*^{(j)}}\mathbb{E}_{\nu, \pi}[C^{(j)}(T)].
$$
This is true as for each collision the agent $j$ obtains $\mu_{jk_*^{(j)}}$ regret ($0$ reward) in expectation, and for each successful play of  arm $k$ it obtains $\Delta^{(j)}_{k}$ regret. 
Therefore, a trivial regret lower bound is 
$$
R^{(j)}_T(\boldsymbol{\nu}, \pi) \geq \sum_{k=1}^{K}\Delta^{(j)}_{k}\mathbb{E}_{\nu, \pi}[N^{(j)}_k(T)].
$$

For an OSB instance, we know that the number of times the agents $1$ to $(j-1)$ plays arm $k_*^{(j)}$ successfully, the agent $j$ should either move to a sub-optimal arm (as the arm $k_*^{(j)}$ is the optimal arm for agent $j$ in an OSB instance) or it is blocked. In the best possible scenario, the agent $j$ successfully plays its second best arm, in each of these instances. This holds as $\Delta^{(j)}_{\min} \leq \mu_{jk_*^{(j)}}$ for non-negative rewards. Therefore, the regret from the events when agents $1$ to $(j-1)$ plays arm $k_*^{(j)}$ successfully, is lower bounded by 
$$
R^{(j)}_T(\boldsymbol{\nu}, \pi) \geq \sum_{j'=1}^{j-1}\Delta^{(j)}_{\min}\mathbb{E}_{\nu, \pi}[N^{(j')}_{k_*^{(j)}}(T)].
$$

Therefore, the combined regret lower bound is given as 
$$
R^{(j)}_T(\boldsymbol{\nu}, \pi) \geq \max\left\{\sum_{k=1}^{K}\Delta^{(j)}_{k}\mathbb{E}_{\nu, \pi}[N^{(j)}_k(T)], \sum_{j'=1}^{j-1}\Delta^{(j)}_{\min}\mathbb{E}_{\nu, \pi}[N^{(j')}_{k_*^{(j)}}(T)]\right\}
$$
\end{proof}

\subsection{Proof of Regret Lower Bound (Theorem~\ref{thm:lower_bound})}
\begin{proof}
We consider  any instance in the class of OSB $\boldsymbol{\nu}$, universally consistent policy $\pi$,   agent $j \in [N]$, and arm $k \in [K] \setminus \{k_*^{(j')}: 1\leq j'\leq j\}$. Let us consider the instance $\boldsymbol{\nu}'$ (which is specific to the $j$ and $k$ pair) where $\nu'_{j'k'} = \nu_{j'k'}$ for all $j'\neq j, k'\neq k$,  $\nu'_{jk}$ 
such that $D(\nu_{jk}, \nu'_{jk}) \leq D_{\inf}(\nu_{jk}, \mu_{jk_*^{(j)}},  \mathcal{P}) + \epsilon$ and $\mu'_{jk}\equiv\mu(\nu'_{jk}) > \mu_{jk_*^{(j)}}$ for some $\epsilon > 0$. Note, for  $\mu_{jk_*^{(j)}} < 1$ and $\Delta^{(j)}_k > 0$, which holds by assumption, the distribution $\nu'_{jk}$ exists by definition of $D_{\inf}(\cdot)$. In short, for the $j$-th agent we make the $k$-th arm optimal. The optimal arm for agent $j$ in the instance $\boldsymbol{\nu}'$ is the arm $k$.

For any event $A$ (and its complement $A^c$), due to Pinsker's inequality we have  
\begin{equation}\label{eq:pinskers}
D(\mathbb{P}_{\nu, \pi}, \mathbb{P}_{\nu', \pi}) \geq \log\left(\tfrac{1}{2(\mathbb{P}_{\nu, \pi}(A) + \mathbb{P}_{\nu, \pi}(A^c))}\right).    
\end{equation}

Let us now consider the event $A = \{N^{(j)}_k(T) \geq T/2\}$. Therefore, due to the regret decomposition lemma~\ref{lemm:regret}, we have the regrets:
\begin{enumerate}
    \item  In instance $\boldsymbol{\nu}$ as $R_{T}(\boldsymbol{\nu}, \pi) \geq R_{T}^{(j)}(\boldsymbol{\nu}, \pi) \geq \Delta^{(j)}_k \tfrac{T}{2} \mathbb{P}_{\nu, \pi}\left(\{N^{(j)}_k(T) \geq T/2\}\right).$
    \item In instance $\boldsymbol{\nu}'$ as $R_{T}(\boldsymbol{\nu}', \pi) \geq R_{T}^{(j)}(\boldsymbol{\nu}', \pi) \geq (\mu'_{jk} - \mu_{jk_j^*}) \tfrac{T}{2} \mathbb{P}_{\nu, \pi}\left(\{N^{(j)}_k(T) < T/2\}\right).$
\end{enumerate}
As the only change in reward distribution happens in agent $j$, arm $k$ pair, we have from Lemma~\ref{lemm:divergence}: 
$$
D(\mathbb{P}_{\nu, \pi}, \mathbb{P}_{\nu', \pi}) =D(\nu_{jk}, \nu'_{jk})  \mathbb{E}_{\nu, \pi}[N^{(j)}_k(T)]  \leq 
\left(\epsilon + D_{\inf}(\nu_{jk}, \mu_{jk_*^{(j)}}, \mathcal{P})\right)\mathbb{E}_{\nu, \pi}[N^{(j)}_k(T)] .
$$
The last inequality holds true by construction of $\nu'_{jk}$.

Substituting the above three relations in Equation~\eqref{eq:pinskers} we obtain for any $\epsilon > 0$.
\begin{align*}
    &\left(\epsilon + D_{\inf}(\nu_{jk}, \mu_{jk_*^{(j)}}, \mathcal{P})\right)\mathbb{E}_{\nu, \pi}[N^{(j)}_k(T)]
    \geq \log\left(\tfrac{1}{2(\mathbb{P}_{\nu, \pi}(A) + \mathbb{P}_{\nu, \pi}(A^c))}\right)
    \geq \log\left(\tfrac{T\min((\mu'_{jk} - \mu_{jk_*^{(j)}}), \Delta^{(j)}_k)}
    {4 (R_{T}^{(j)}(\boldsymbol{\nu}, \pi) + R_{T}^{(j)}(\boldsymbol{\nu}', \pi))}\right).
\end{align*}
Here, the final inequality hold as the policy $\pi$ is assumed to be universally consistent.
Therefore, taking the following holds after taking the 
\begin{align*}
    \lim\limits_{\epsilon \to 0}\liminf\limits_{T\to \infty} \frac{\mathbb{E}_{\nu, \pi}[N^{(j)}_k(T)]}{\log T} 
    \geq \lim\limits_{\epsilon \to 0} \frac{1}{\epsilon + D_{\inf}(\nu_{jk}, \mu_{jk_*^{(j)}}, \mathcal{P})} 
    = \frac{1}{D_{\inf}(\nu_{jk}, \mu_{jk_*^{(j)}}, \mathcal{P})}.
\end{align*}

As the above bound is true for any uniformly consistent policy $\pi$, and for agent $j$, and arm $k \in [K] \setminus \{k_*^{(j')}: 1\leq j'\leq j\}$. We use the regret decomposition lemma (Lemma\,\ref{lemm:regret}) to obtain the final aysmptotic regret lower bound for any agent $j\in [N]$ as  
$$
\liminf\limits_{T\to \infty} \frac{R_T^{(j)}(\boldsymbol{\nu})}{\log T} 
\geq \max\left\{\sum_{j'=1}^{j-1}  \frac{ \Delta^{(j)}_{\min}}{D_{\inf}(\nu_{j'k_*^{(j)}}, \mu_{j'k_*^{(j')}}, \mathcal{P})}, 
\sum_{k\notin \mathcal{A}_*^{(j)}\setminus k_*^{(j)}} \frac{\Delta^{(j)}_k}{D_{\inf}(\nu_{jk}, \mu_{jk_*^{(j)}}, \mathcal{P})}\right\}
$$
Here, we use the fact that $k_*^{(j)} \notin \mathcal{D}_*^{(j')} \cup \{k_*^{(j')}\}$ for all $j'<j$. Also note, $\sum_{j'=1}^{0} (\cdot) = 0$ and $\mathcal{D}_*^{(1)} = \emptyset$ for the highest ranked arm. 
\end{proof}

\subsection{Proof of Corollary~\ref{corr:simple}}
The above corollary follows readily from Theorem~\ref{thm:lower_bound}. Let for agent $j'$ from $1$ to $j-1$ the optimal arm be $j'$ with mean $1/2$ and all the other arms have mean $1/2 - \Delta$, where $\Delta > 0$ is small enough. Also, let the $j$-th agent have the arm means between $1/2$ for the $j$-th arm and $1/4$ for any other arm. For $\mathcal{P}$ the class of Bernoulli rewards, we have $D_{\inf}(\nu_{j'k_*^{(j)}}, \mu_{j'k_*^{(j')}}, \mathcal{P}) \leq \Delta^2/4$ for all $j'\leq j-1$, and $\Delta^{(j)}_{\min} = 1/4$. Therefore, the regret of the $j$-th agent is lower bounded as $\frac{(j-1)\log(T)}{16\Delta^2}$.

\section{Additional Simulations}
\label{appendix-simulations}
In this section, we compare our algorithm to both, the ETC based decentralized algorithm and the centralized UCB. The main conclusion is that, in both small and large systems, our algorithm outperforms the prior decentralized ETC algorithm \cite{matching_bandits-old} and is comparable to the centralized UCB algorithm of \cite{matching_bandits}.

\textbf{Simulation Setup}
We consider $4$ systems -- the first two systems are the OSB systems with $10$ agents and $10$ arms (Figure \ref{fig:10-10-system}), and $10$ agents, $15$ arms (Figure \ref{fig:10_15_system}). 
In both these systems, a random permutation $\sigma$ was first chosen and the arm-mean of arm $\sigma(i)$ for agent $i$ was set to $0.9$. All other arm-means was chosen randomly and uniformly in $[0,0.8]$. We then consider two non OSB systems with $5$ agents and $7$ arms (Figure \ref{fig:5_7_system}) and $10$ agents and $15$ arms (Figure \ref{fig:10_10_system_nonosb}). In these two systems, every agent $j$, uniformly spaces the arm-means between $0.1$ and $0.9$, with each agent having a random permutations over the arms to arrange the arm-means.
%The instances (arm-means) with which we conduct simulations is in the supplementary materials. 
All plots are averaged over $30$ trials with confidence intervals of $95\%$. For brevity, Figures \ref{fig:10-10-system},\ref{fig:10_15_system} and \ref{fig:10_10_system_nonosb} are in the Appendix.

% The simulations consists of two systems -- one with $10$ agents and $10$ arms (Figure \ref{fig:10-10-system} ) and another system consisting of $10$ agents and $15$ arms (Figures \ref{fig:10_15_system}).
% In both systems, the arm-means were randomly chosen with
% the arm-means for any agent $i \in [10]$, for all arms except $i$, chosen i.i.d. uniform in $[0,0.8]$. The arm-mean of arm $i$ for agent $i$ was chosen to be $0.9$. 
% The exact instance used is supplied in the additional supplementary materials.
% Thus, the unique stable matching in this case consists of matching every agent $i \in [10]$ to arm $i$.
% The rewards are binary valued.
% To compare against the ETC algorithm, we used a parameter of $H=1117$ for the system with $10$ agents and $15$ arms and $H = 805$ for the system with $10$ agents and $15$ arms.
%  We consider a time horizon of $T^5$.
% The results are plotted after averaging over $10$ runs and the  intervals denote $95\%$ confidence.
% We compare our algorithm against both, the centralized UCB and the decentralized ETC from \cite{matching_bandits}. 
% For the ETC algorithm of \cite{matching_bandits}, we set the parameter of $H= 10^3$.

\textbf{Comparison with other Algorithms} -
In Figures \ref{fig:10_10_regret}, \ref{fig:10_15_regret}, \ref{fig:5_7_regret} and \ref{fig:10_10_regret_nonosb}, we plot the regret of all agents, for the three algorithms. We observe that \name \ outperforms ETC and is slightly poorer compared to the centralized UCB algorithm.
The centralized UCB has no collisions as a central arbiter matches agents and arms in the centralized UCB, and thus the regret is expected to be lower than any decentralized algorithm, which incurs some collisions. Figures \ref{fig:collision_plot_10_10}, \ref{fig:collision_plot_10_15}, \ref{fig:collision_plot_5_7}, \ref{fig:collision_plot_10_10_nonosb} highlight this, where we plot the regret incurred by all algorithms only on account of collisions.
The collisions incurred in our algorithm are lower compared to the decentralized ETC algorithm, thereby incurring lower regret compared to ETC. 
Although for a few high ranked agents, ETC has lower collisions (Fig. \ref{fig:collision_plot_10_15}), the overall regret of agents is lower with \name \ algorithm as opposed to ETC.
The deletion of dominant arms plays a key role, which enables our algorithm to have reduced collisions and thus lower regret.

\textbf{Equilibrium Freezing of \name \ } -
In Figures \ref{fig:10_10_heatmap}, \ref{fig:10_15_heatmap}, \ref{fig:5_7_heatmap} and \ref{fig:10_10_heatmap_nonosb}, we plot a `heatmap' of the arms recommended by the agents over the $13$ phases. 
The darker the shade, the higher the frequency (over the different simulation runs), that a particular agent recommended a particular arm in a particular phase.
We observe from Figure \ref{fig:10_15_heatmap} that  after a random phase, all agent always recommend their stable match partner arm.
Moreover, the time for an agent to settle into the `equilibrium' of always recommending their estimated stable match partner arm is larger for lower ranked agents. 
Nevertheless, Figure \ref{fig:10_15_heatmap} shows that after a random time, the agents delete their dominated arms
thereby ``freezing the system into an equilibrium".

%\section{Plots and Figures from Section \ref{appendix-simulations}}
%\label{appendix-figures}
\begin{figure}
\centering
\begin{subfigure}{\linewidth}
    \includegraphics[width=\linewidth]{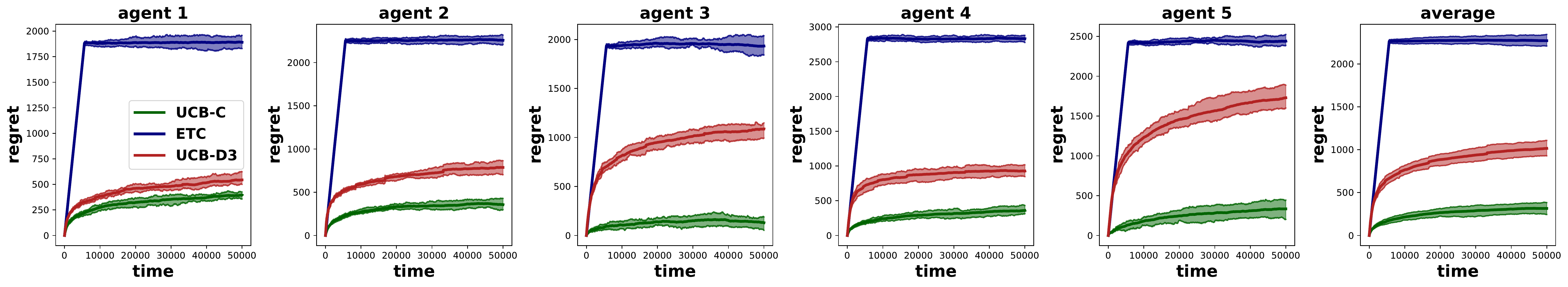}
\caption{Regret plots of all agents}
\label{fig:5_7_regret}
\end{subfigure}\\
\begin{subfigure}{\linewidth}
    \centering
    \includegraphics[width=\linewidth]{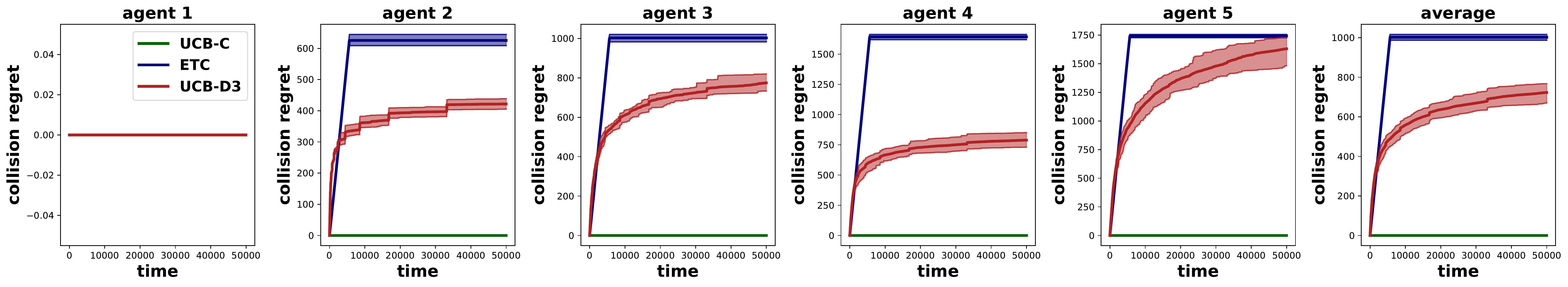}
    \caption{A plot showing the cumulative regret only due to collisions. The centralized UCB ensures that agents never collide and thus do not lose out on regret.}
    \label{fig:collision_plot_5_7}
\end{subfigure}\\
\begin{subfigure}{\linewidth}
\centering
  \includegraphics[width=\linewidth]{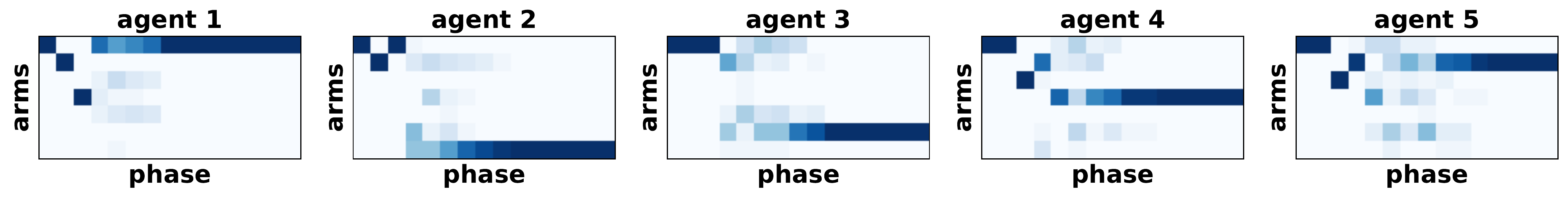}
\caption{Arms recommended by the agents across phases over different runs of the algorithm.}
\label{fig:5_7_heatmap}
\end{subfigure}
\caption{Simulations on a system with $5$ agents and $7$ arms. For each agent $i\in[5]$, a permutation over the arms $\sigma_i$ was chosen, and the arm-means are equally spaced among the $7$ arms from $0.1$ to $0.9$ in the increasing order of permutation. This is thus not a OSB instance. The rewards are binary. The value of $H=801$ was used for ETC.}
\label{fig:5_7_system}
\end{figure}

\begin{figure}
\centering
%\begin{tabular}[c]
\begin{subfigure}{\linewidth}
\centering
    \includegraphics[width=\linewidth]{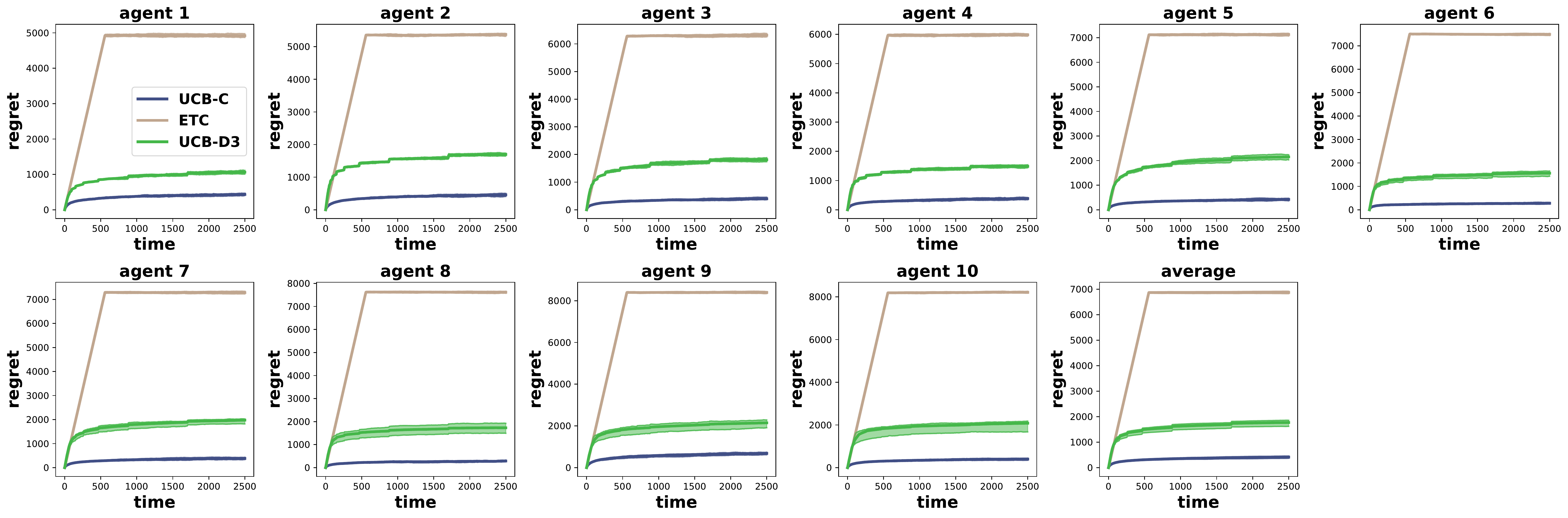}
\caption{Regret plots of all agents.}
\label{fig:10_10_regret}
\end{subfigure}\\
\begin{subfigure}{\linewidth}
  \centering
    \includegraphics[width=\linewidth]{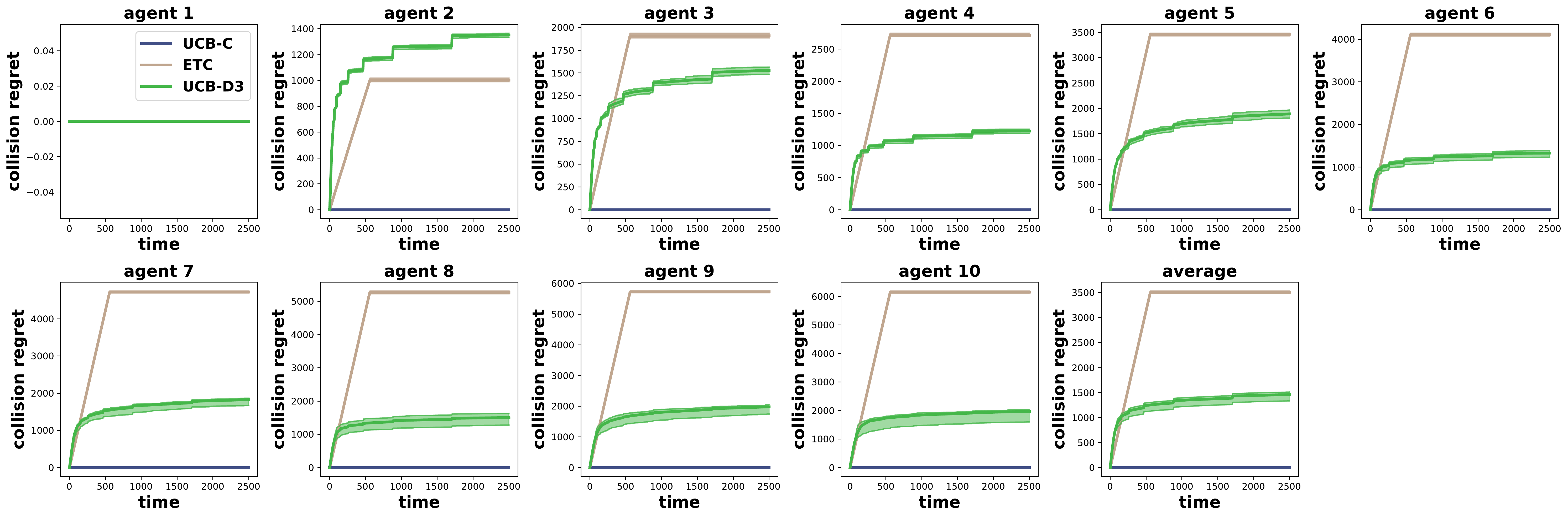}
    \caption{A plot showing the cumulative regret only due to collisions. The centralized UCB ensures that agents never collide and thus do not lose out on regret.}
  \label{fig:collision_plot_10_10}
\end{subfigure}\\
\begin{subfigure}{\linewidth}
\centering
  \includegraphics[width=\linewidth]{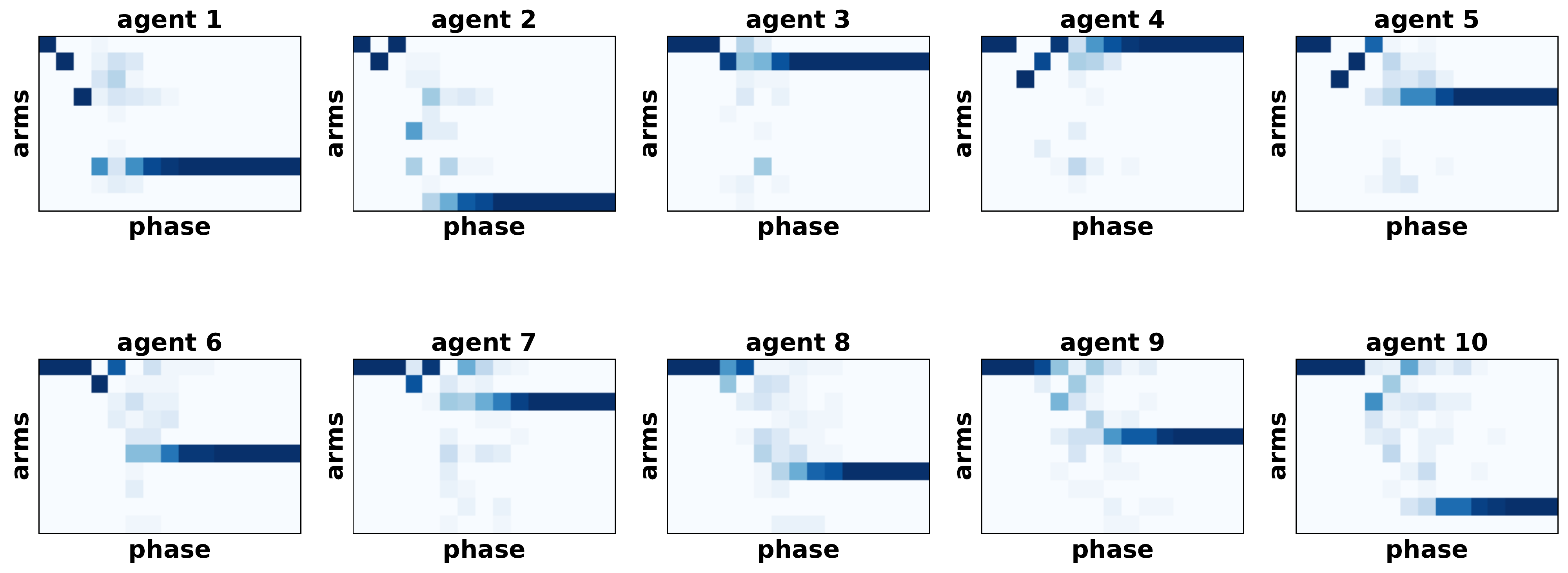}
\caption{Arms recommended by the agents across phases over different runs of the algorithm.}
\label{fig:10_10_heatmap}
\end{subfigure}
\caption{Simulations on a system with $10$ agents and $10$ arms. The arm-means for sub-optimal arms for each agent are chosen i.i.d. uniformly over $[0,0.8]$, while the arm-mean of agent $i \in [10]$ for arm $\sigma(i)$ (its optimal stable match arm) was set to $0.9$. The rewards are binary. Here, $\sigma(\cdot)$ denotes a permutation. This is thus a OSB instance. The value of $H = 1117$ used for ETC.}
\label{fig:10-10-system}
\end{figure}

\begin{figure}
\centering
\begin{subfigure}{\linewidth}
    \includegraphics[width=\linewidth]{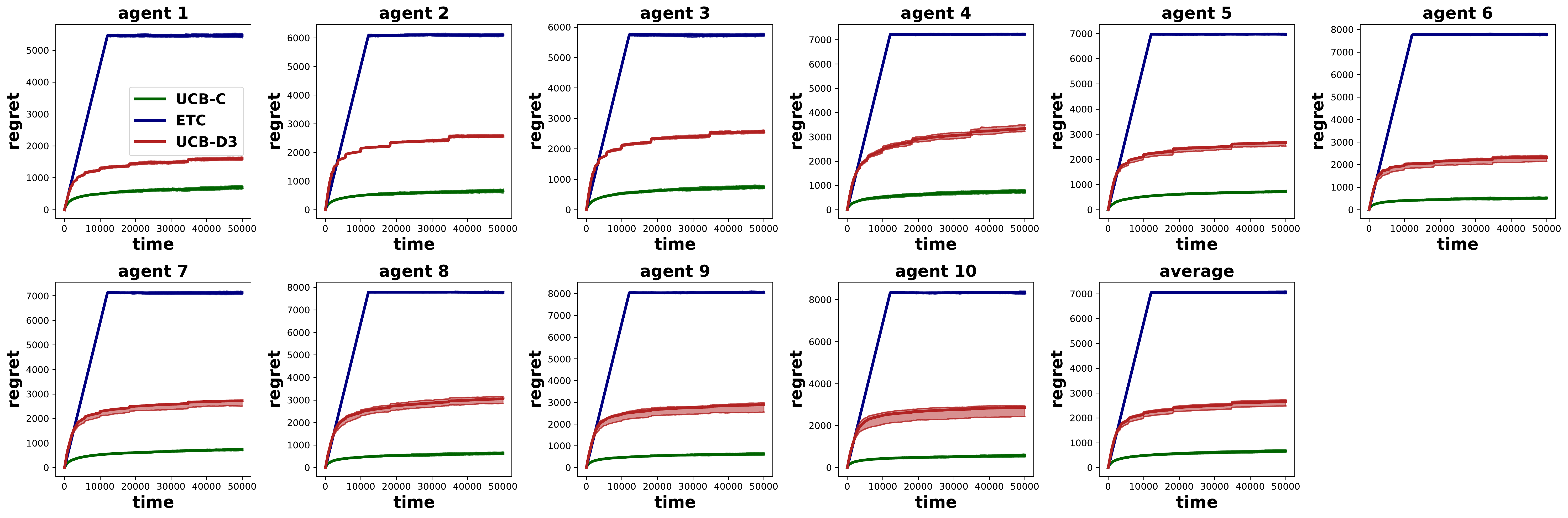}
\caption{Regret plots of all agents}
\label{fig:10_15_regret}
\end{subfigure}\\
\begin{subfigure}{\linewidth}
    \centering
    \includegraphics[width=\linewidth]{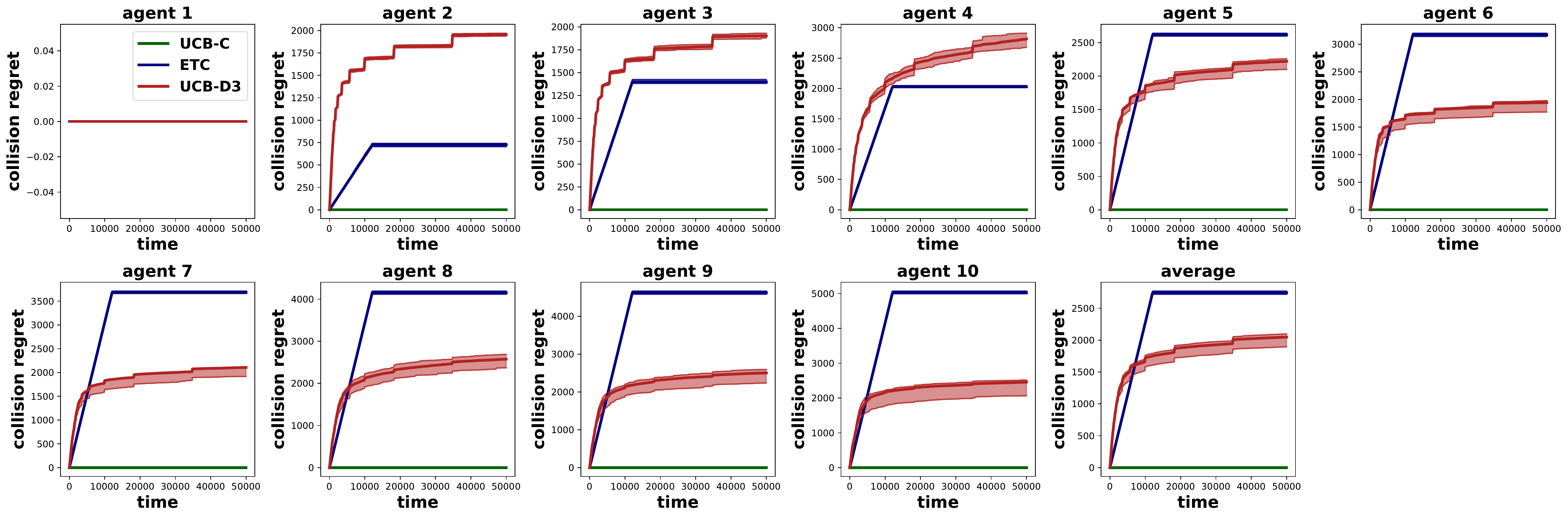}
    \caption{A plot showing the cumulative regret only due to collisions. The centralized UCB ensures that agents never collide and thus do not lose out on regret.}
    \label{fig:collision_plot_10_15}
\end{subfigure}\\
\begin{subfigure}{\linewidth}
\centering
   \includegraphics[width=\linewidth]{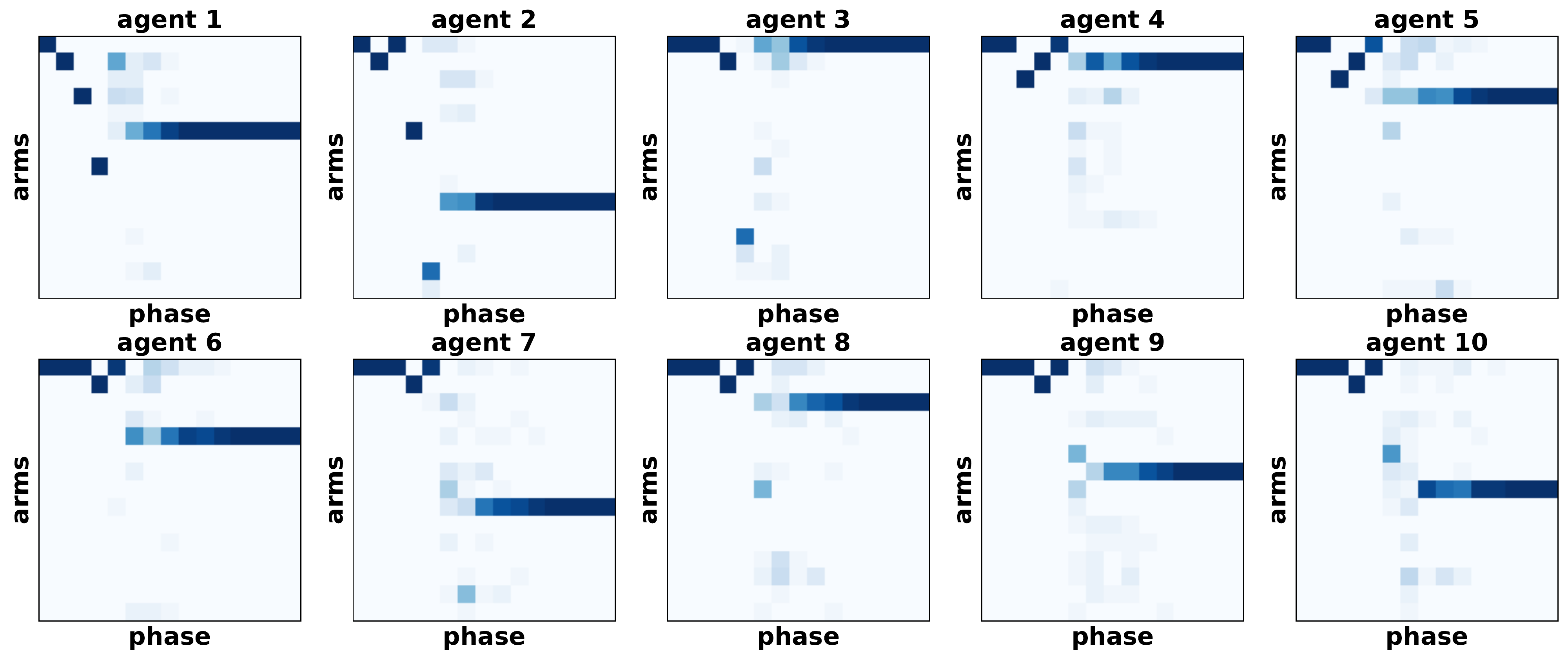}
\caption{Arms recommended by the agents across phases over different runs of the algorithm.}
\label{fig:10_15_heatmap}
\end{subfigure}
\caption{Simulations on a system with $10$ agents and $15$ arms. The arm-means for sub-optimal arms for each agent are chosen i.i.d. uniformly over $[0,0.8]$, while the arm-mean for agent $i \in [10]$ and arm $\sigma(i)$ (stable match partner arm) was set to $0.9$. The rewards are binary. Here, $\sigma(\cdot)$ denotes a permutation. This is thus a OSB instance. The value of $H=805$ was used for ETC.}
\label{fig:10_15_system}
\end{figure}

% \begin{figure}
% \centering
% \begin{subfigure}{\linewidth}
%  % \includegraphics[width=\linewidth]{Figures/regret_15_arms_10_agents.png}
%     \includegraphics[width=\linewidth]{5x7-regret-new.pdf}
% \caption{Regret plots of all agents}
% \label{fig:5_7_regret}
% \end{subfigure}\\
% \begin{subfigure}{\linewidth}
%     \centering
%     \includegraphics[width=\linewidth]{5x7-collision.pdf}
%     \caption{A plot showing the cumulative regret only due to collisions. The centralized UCB ensures that agents never collide and thus do not lose out on regret.}
%     \label{fig:collision_plot_5_7}
% \end{subfigure}\\
% \begin{subfigure}{\linewidth}
% \centering
%   %\includegraphics[width=\linewidth]{Figures/heatmap_15_arms_10_agents.png}
%   \includegraphics[width=\linewidth]{5x7-phases.pdf}
% \caption{Arms recommended by the agents across phases over different runs of the algorithm.}
% \label{fig:5_7_heatmap}
% \end{subfigure}
% \caption{Simulations on a system with $5$ agents and $7$ arms. For each agent $i\in[5]$, a permutation over the arms $\sigma_i$ was chosen, and the arm-means are equally spaced among the $7$ arms from $0.1$ to $0.9$ in the increasing order of permutation. This is thus not a OSB instance. The rewards are binary. The value of $H=801$ was used for ETC.}
% \label{fig:5_7_system}
% \end{figure}

\begin{figure}
\centering
\begin{subfigure}{\linewidth}
    \includegraphics[width=\linewidth]{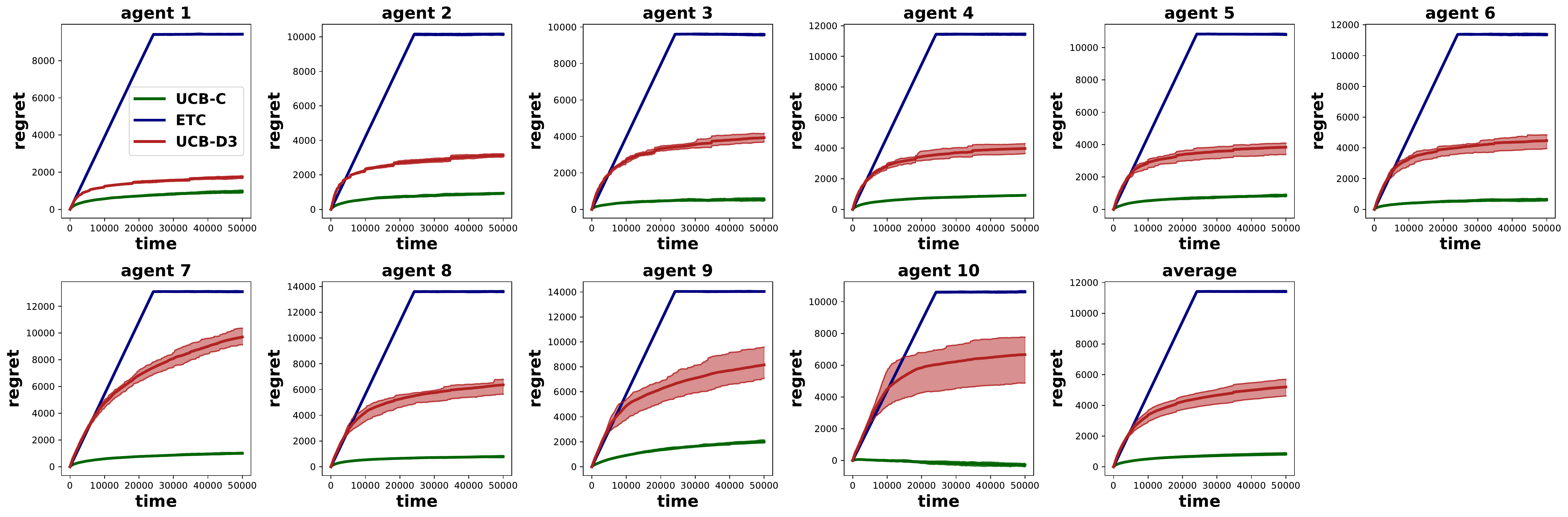}
\caption{Regret plots of all agents}
\label{fig:10_10_regret_nonosb}
\end{subfigure}\\
\begin{subfigure}{\linewidth}
    \centering
    \includegraphics[width=\linewidth]{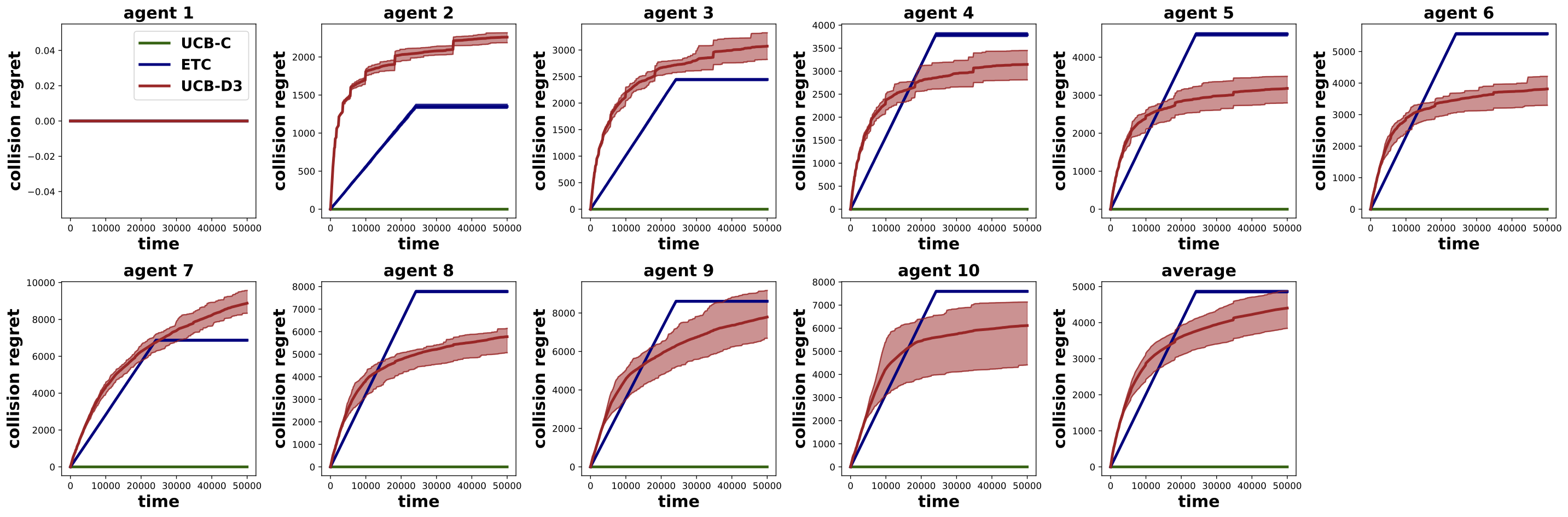}
    \caption{A plot showing the cumulative regret only due to collisions. The centralized UCB ensures that agents never collide and thus do not lose out on regret.}
    \label{fig:collision_plot_10_10_nonosb}
\end{subfigure}\\
\begin{subfigure}{\linewidth}
\centering
   \includegraphics[width=\linewidth]{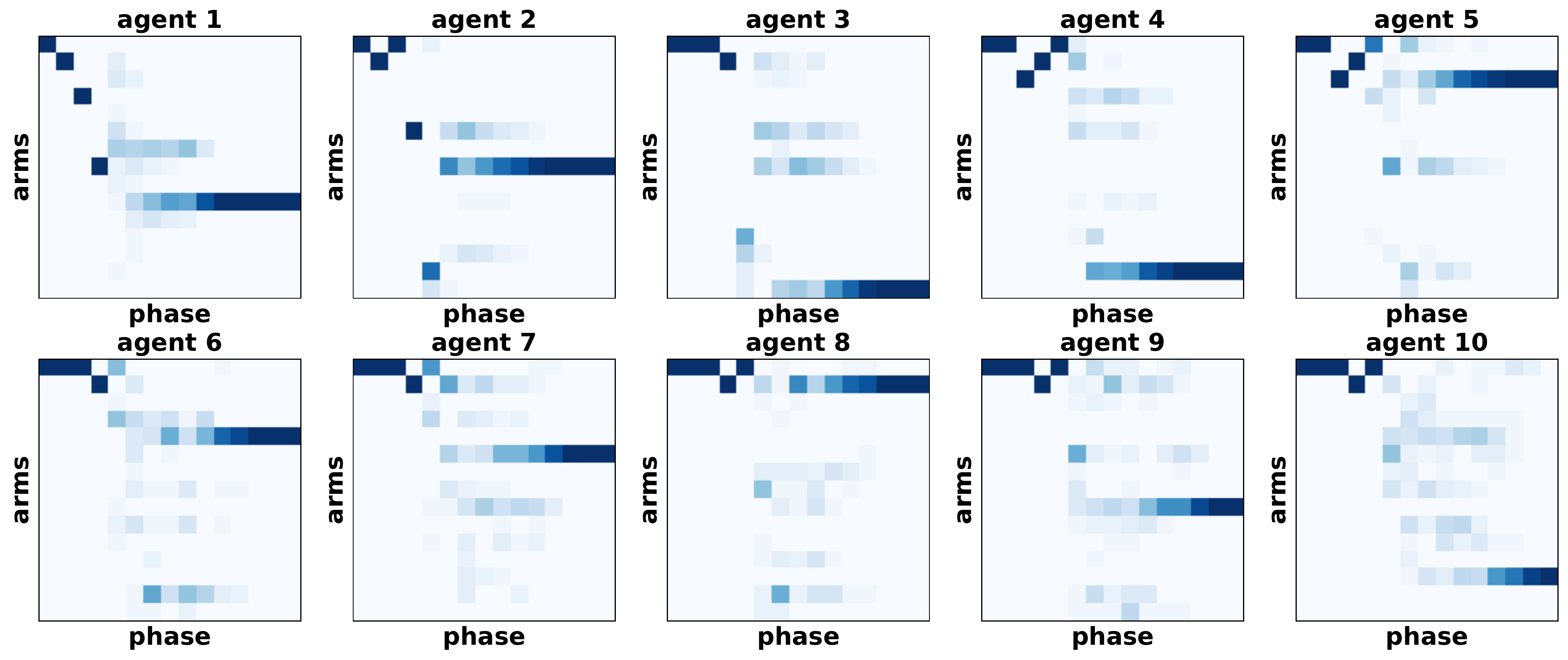}
\caption{Arms recommended by the agents across phases over different runs of the algorithm.}
\label{fig:10_10_heatmap_nonosb}
\end{subfigure}
\caption{Simulations on a system with $10$ agents and $15$ arms. For each agent $i\in[10]$, a permutation over the arms $\sigma_i$ was chosen, and the arm-means are equally spaced among the $7$ arms from $0.1$ to $0.9$ in the increasing order of permutation. This is thus not a OSB instance. The rewards are binary. The value of $H=1610$ was used for ETC.}
\label{fig:10_10_system_nonosb}
\end{figure}

\end{document}